\documentclass{article}

\usepackage[preprint]{neurips_2026}

\usepackage[utf8]{inputenc} 
\usepackage[T1]{fontenc}    
\usepackage{hyperref}       
\usepackage{url}            
\usepackage{booktabs}       
\usepackage{amsfonts}       
\usepackage{nicefrac}       
\usepackage{microtype}      
\usepackage{xcolor}         

\usepackage{amsmath}
\usepackage{amssymb}
\usepackage{mathtools}
\usepackage{amsthm}

\usepackage{algorithmic}
\usepackage{algorithm}

\theoremstyle{plain}
\newtheorem{theorem}{Theorem}[section]

\newtheorem{lemma}[theorem]{Lemma}
\newtheorem{corollary}[theorem]{Corollary}
\theoremstyle{definition}

\theoremstyle{remark}

\DeclareFontFamily{U}{stix2bb}{}
\DeclareFontShape{U}{stix2bb}{m}{n} {<-> stix2-mathbb}{}
\NewDocumentCommand{\indicator}{}{\text{\usefont{U}{stix2bb}{m}{n}1}}

\title{Safe Bayesian Optimization\\with Counterfactual Policies}

\author{%
  Katherine Avery \\
  College of Computer Science\\
  University of Massachusetts Amherst\\
  Amherst, MA 01002 \\
  \texttt{kavery@cs.umass.edu} \\
  \And
  Bruno Castro da Silva \\
  College of Computer Science\\
  University of Massachusetts Amherst\\
  Amherst, MA 01002 \\
  \texttt{bsilva@cs.umass.edu}
  \And
  David Jensen \\
  College of Computer Science\\
  University of Massachusetts Amherst\\
  Amherst, MA 01002 \\
  \texttt{jensen@cs.umass.edu}
}

\begin{document}

\maketitle

\begin{abstract}
In many decision-making settings, new interventions are acceptable only if they do not reduce outcomes below some established threshold. For example, in clinical medicine, new treatments are often acceptable only if they do not worsen outcomes relative to an established standard of care. Safe Bayesian optimization maximizes an objective subject to safety constraints. In the setting that we consider here, safety is defined relative to a known baseline policy whose outcomes are counterfactual and therefore unobserved. Thus, the counterfactual outcomes of the baseline policy must be estimated and those (uncertain) estimates must be used to safely optimize the objective. We address this estimation problem by using conformal prediction to construct valid uncertainty intervals for counterfactual baseline outcomes, and we show how these intervals can be integrated into safe Bayesian optimization to ensure that constraint violations occur at or below a user-specified rate. We also show how to adapt these conformal estimates to different kinds of covariate shift. We provide a safety proof, experimental evidence, and a sensitivity analysis. 
\end{abstract}

\section{Introduction}

In many sequential decision-making problems, we seek to improve performance while ensuring that each action is safe. We may define safety in absolute terms, such as keeping blood pressure below a threshold, but we may instead define it relative to an established baseline. For example, when determining the effectiveness of a new drug for cancer treatment, we may not want patients’ quality of life to be substantially worse than it would have been under the pre-existing standard of care. This type of safety constraint is challenging since the patients' quality of life when taking the standard drug is not observed and can only be estimated. 

We frame this problem using the framework of safe Bayesian optimization (SafeOpt),\footnote{SafeOpt sometimes refers to the specific method proposed in \citeauthor{sui2015} \citeyearpar{sui2015} and sometimes refers to safe Bayesian optimization in general. In this work, it refers to the latter.} which optimizes some objective function $f(\mathbf{x})$, such as information gain, with respect to some constraint: 
\begin{equation} \label{eq:constrained-opt}
    \max_{\mathbf{x}\in\mathcal{X}} f(\mathbf{x}) \text{ s.t. } q(\mathbf{x}) \ge 0.
\end{equation}
SafeOpt optimizes the objective within a set of safe states $\mathcal{S}$ (where $\mathcal{X}\subseteq\mathcal{S}$) and expands $\mathcal{S}$ as more information is gathered. We focus specifically on SafeOpt with online conformal prediction \citep{zhang2024}, which provides distribution-free safety guarantees. Conformal SafeOpt ensures that the rate that $q(\mathbf{x}) \ge 0$ is violated over $T$ optimization steps is at or below $\alpha$, where $\alpha$ is user-specified. See Sec.~\ref{sec:background-cp}. 

Potential safe $\mathbf{x}$ values are typically determined based on the values of $q$ that have been seen so far. However, we are interested in safety constraints that contain counterfactuals. For example, we may want the value of a variable $Y$ under a treatment minus the counterfactual $Y$ under a standard-of-care treatment to be at least some value $-\omega$. Because $q$ depends on a counterfactual that is not observed, it is difficult to provide typical SafeOpt-style safety guarantees. Simply making a point estimate of the counterfactual can result in a violate rate above $\alpha$. 

Therefore, we use split conformal prediction to estimate the outcome of the counterfactual $\mathbf{x}$ under a standard of care \citep{lei2021}. Split conformal prediction creates interval estimates, rather than point estimates. There is a $100(1-\epsilon)$\% chance that the true value of the counterfactual falls within the interval, where $\epsilon$ is user-specified. $\epsilon$ can be leveraged to prove the formal safety guarantee required by SafeOpt methods (proof sketch in Sec.~\ref{sec:proof} and full proof in Appx.~\ref{appx:proofs}). 

We perform our primary experiments on two settings, both used in \citeauthor{zhang2024} \citeyearpar{zhang2024}. The first uses a chemical reaction simulator \citep{kang2019}, which gives us access to the true counterfactual outcomes. The second is the MovieLens dataset \citep{harper2015}. In both cases, we use the method in Sec.~\ref{sec:methodology} to estimate the constraint values for different constraints with counterfactuals. Our baselines include standard (non-safe) Bayesian optimization and an oracle with full knowledge of the counterfactual. We find that our method optimizes the objective while meeting the safety requirement that is provided in Sec.~\ref{sec:proof}. 

In summary, our contributions include the following. We create a method that ensures that the actions taken by safe Bayesian optimization do not perform $\omega$ worse than some known, safe policy for $100(1-\alpha)$\% of samples, where $\omega$ and $\alpha$ are user-specified. In Sec.~\ref{sec:methodology} and Appx.~\ref{appx:nonstationarity}, we show that the method is robust to different kinds of covariate shifts when re-weighting the calibration set. We prove that the method meets safety guarantees under minimal assumptions, and we bound the conformal coverage gap under assumption violations in Appx.~\ref{appx:coverage-gap}. Sec.~\ref{sec:results} and Appx.~\ref{appx:additional-results} provide experimental evidence that the safety constraint is met for $100(1-\alpha)$\% of samples and that the proposed method can adapt to different kinds of covariate shifts. We show this for constraints with one and two estimated terms. Finally, we provide a sensitivity analysis to examine how the method performs under different kinds and amounts of mis-specification.

\subsection{Motivating example}\label{sec:example}
Suppose you are running a clinical trial that tests a new drug for reducing systolic blood pressure. You give patients treatments one after the other for $T$ timesteps. One of the many safety constraints in your study compares the reduction in blood pressure under the new treatment $Y(\mathbf{x}_t)$ at time $t$ with the reduction under the standard-of-care policy $Y(\mathbf{x}^{soc}_t)$. You want $Y(\mathbf{x}_t)$ to be at least as good as $Y(\mathbf{x}^{soc}_t)$ with some tolerance level $\omega$. Because you cannot ensure that the constraint will never be violated, you specify a small violation rate $\alpha=0.01$. That is, the safety constraint can be violated $1\%$ of the time over $T$ timesteps. 

You first gather data from the standard-of-care policy. The data includes your treatment, any covariates, and the blood pressure outcome under the treatment. You split the data into a training set $\mathcal{D}_{train}$ and a calibration set $\mathcal{D}_{cal}$. $\mathcal{D}_{train}$ is used to get a point estimator for $Y(\mathbf{x}_{soc})$, and $\mathcal{D}_{cal}$ is used to create quantiles. You then start generating new treatments subject to the safety constraint using SafeOpt. For every new treatment that you generate, you estimate the upper bound of the counterfactual $Y(\mathbf{x}_{soc})$ and use that to compute your safety constraint at each timestep $q(\mathbf{x}_t)$.

Sec.~\ref{sec:methodology} lays out the methodology in more detail. We also describe how to adapt the method to different scenarios by re-weighting the calibration data to look like it came from the standard-of-care policy with the test covariates observed during SafeOpt. For example, if the training and calibration data are observational, the calibration data can be re-weighted to look like it came from the standard-of-care policy (Scenario 1). Alternatively, suppose there is covariate shift. The training and calibration data are collected at Hospital A. SafeOpt is run at Hospital B, which serves an older population. By observing the ages of the people in the study, we can re-weight the calibration data (Scenario 2).

\section{Related work}

\subsection{Safe Bayesian optimization under noisy constraint observations}\label{sec:related-bo}

Safe Bayesian optimization \citep{sui2015, berkenkamp2023, zhang2024} imposes safety constraints on standard Bayesian optimization, as in Eq.~\ref{eq:constrained-opt}. Our work considers safety constraints with estimated components, implying that the true value of the safety constraint is never actually observed. \cite{bergmann2020} considers SafeOpt with constraint observations that may be corrupted by a known amount of Gaussian measurement noise and reformulates the safety constraint to be probabilistic. In addition to the deterministic setting laid out in Sec. 5, \cite{zhang2024} considers the case when observations of the constraint have been corrupted by a known amount of measurement noise. They reformulate their probabilistic constraint to account for the noise. Other work in constrained Bayesian optimization reformulates hard constraints as probabilistic constraints under noisy constraint observations \citep{gelbart2014} or deals with the noise through better modeling of the constraint \citep{letham2019, wang2024}. Our work is distinct since we specifically consider safety constraints under estimation error for a variety of settings. 

\subsection{Constrained multi-armed bandits}

Safe Bayesian optimization is similar to constrained bandits. For example, \cite{amani2019} and \cite{moradipari2021} learn a policy while taking actions subject to a safety constraint by maintaining a safe set, which is similar to SafeOpt. Other kinds of constraints include budget constraints \citep{zhou2018,agrawal2016,agrawal2014,wu2015} and expectation constraints \citep{pacchiano2021}. In contrast to other work, our setting uses partially observed constraints. 

\section{Background}\label{sec:background}
In safe Bayesian optimization, we maximize an objective function $f(\mathbf{x})$ while obeying a safety constraint, as in Eq.~\ref{eq:constrained-opt}. We assume that an optimal solution to Eq.~\ref{eq:constrained-opt} exists, and we assume that the initial set of safe candidate solutions $\mathcal{S}_0$ is not empty. All decisions made by the system fall within the safe set, so $\mathcal{S}_0$ needs to be nonempty: $\mathcal{S}_0 \subseteq \{ \mathbf{x}\in\mathcal{X}: q(\mathbf{x}) \ge 0 \}$.

Both the safe set and the Bayesian posterior for the objective function are updated at each timestep $t$. We can do this by creating credible intervals for $f$ and $q$, which use Bayesian priors to estimate the ranges of possible values for $f$ and $q$ \citep{zhang2024, sui2015}. \citeauthor{zhang2024} \citeyearpar{zhang2024} defines the intervals using Gaussian processes with mean $\mu$ and standard deviation $\sigma$ as follows: 
\begin{equation}\label{eq:credible-f}
\begin{split}
    \mathcal{I}_f(\mathbf{x}|\mathcal{O}_t) =& [ f_l(\mathbf{x}|\mathcal{O}_t), f_u(\mathbf{x}|\mathcal{O}_t) ] \\
    =&[\mu_f(\mathbf{x}|\mathbf{X}_t, \mathbf{f}_t) - \beta_{t+1}\sigma_f(\mathbf{x}|\mathbf{X}_t, \mathbf{f}_t), 
    \mu_f(\mathbf{x}|\mathbf{X}_t, \mathbf{f}_t) + \beta_{t+1}\sigma_f(\mathbf{x}|\mathbf{X}_t, \mathbf{f}_t)]
\end{split}
\end{equation}
\begin{equation}\label{eq:credible-q}
\begin{split}
    \mathcal{I}_q(\mathbf{x}|\mathcal{O}_t) =& [ q_l(\mathbf{x}|\mathcal{O}_t), q_u(\mathbf{x}|\mathcal{O}_t) ] \\
    =&[\mu_q(\mathbf{x}|\mathbf{X}_t, \mathbf{q}_t) - \beta_{t+1}\sigma_q(\mathbf{x}|\mathbf{X}_t, \mathbf{q}_t), 
    \mu_q(\mathbf{x}|\mathbf{X}_t, \mathbf{q}_t) + \beta_{t+1}\sigma_q(\mathbf{x}|\mathbf{X}_t, \mathbf{q}_t)]
\end{split}
\end{equation}
where $\mathcal{O}_t=(\mathbf{X}_t, \mathbf{f}_t, \mathbf{q}_t)$ is the history up until time $t$. $\mathbf{X}_t$ is the decision history $(\mathbf{x}_1, ..., \mathbf{x}_t)$. $\mathbf{f}_t$ is the history of $f(\mathbf{x})$: $(f(\mathbf{x}_1), ..., f(\mathbf{x}_t))$. Similarly, $\mathbf{q}_t$ is the history of $q(\mathbf{x})$: $(q(\mathbf{x}_1), ..., q(\mathbf{x}_t))$. $\beta_{t+1}$ controls the width of the interval and is defined in Sec.~\ref{sec:background-cp}. 

We then choose an $\mathbf{x}_t$ that optimizes the objective function and/or expands the safe set by shrinking the credible interval of either $f$ or $q$. We do this by finding the action most decreases our uncertainty about $f$ or $q$ (whichever is more uncertain) while maximizing $f$ or expanding the safe set. $\mathcal{M}_{t+1}$ is the set of actions that maximizes $f$, and $\mathcal{G}_{t+1}$ is the set of actions that could expand the safe set:
\begin{equation}\label{eq:x_t1}
    \mathbf{x}_{t+1} = \underset{\mathbf{x}\in \mathcal{M}_{t+1}\cup\mathcal{G}_{t+1}}{\arg\max} \max\{ \sigma_f(\mathbf{x}|\mathcal{O}_t), \sigma_q(\mathbf{x}|\mathcal{O}_t) \}.
\end{equation}
The set of optimizers $\mathcal{M}_{t+1}$ includes all safe solutions that could perform better than the best known pessimistic estimate:
\begin{equation}\label{eq:optimizers}
    \mathcal{M}_{t+1} = \{ \mathbf{x}\in\mathcal{S}_{t+1} | f_u(\mathbf{x}|\mathcal{O}_t) \ge \max_{\mathbf{x}'\in\mathcal{S}_{t+1}} f_l(\mathbf{x}'|\mathcal{O}_t) \}.
\end{equation}
$\mathcal{G}_{t+1}$ is the set of solutions that could expand the safe set. That is, we define $\mathcal{G}_{t+1}$ such that our optimistic prediction of the next safe set $\Tilde{\mathcal{S}}_{t+2}$ is larger than $\mathcal{S}_{t+1}$. $\Tilde{\mathcal{S}}_{t+2}$ includes both the current history $\mathcal{O}_t$ and the optimistic observation $(\mathbf{x}, q_u(\mathbf{x}|\mathcal{O}_t))$: $ \Tilde{\mathcal{S}}_{t+2}(\mathbf{x}) = \mathcal{S}(\mathcal{O}_t \cup (\mathbf{x}, q_u(\mathbf{x}|\mathcal{O}_t))|\beta_{t+1} )$.
\begin{equation}\label{eq:expanders}
    \mathcal{G}_{t+1} = \{ \mathbf{x}\in\mathcal{S}_{t+1}: | \Tilde{\mathcal{S}}_{t+2}(\mathbf{x}) \backslash \mathcal{S}_{t+1} |>0 \}
\end{equation}
\subsection{Safe optimization with conformal prediction}\label{sec:background-cp}
\citeauthor{zhang2024} \citeyearpar{zhang2024} shows how to apply online conformal prediction \citep{gibbs2021, feldman2023} to safe Bayesian optimization. In short, the parameter $\beta_{t+1}$ in Eq.~\ref{eq:credible-f} and~\ref{eq:credible-q} is used to define the scale of the safe set so that it becomes more conservative when the safety constraint violation rate is too high. The violation rate is 
\begin{equation}\label{eq:violation-rate}
    \text{violation-rate}(T) := \frac{1}{T}\sum^T_{t=1} \text{err}_t \le \alpha,
\end{equation}
where $\alpha$ is the acceptable, user-defined violation rate, and the error at each $t$ is defined as
\begin{equation}
    \text{err}_t = \indicator(q(\mathbf{x}_t) < 0).
\end{equation}
$\beta$ is updated as follows:
\begin{equation}\label{eq:beta_t1}
\beta_{t+1} = \varphi(\Delta\alpha_{t+1}),
\end{equation}
where $\Delta\alpha_{t+1}$ is the excess violation rate, which measures the average number of errors above the allowed violation rate $\alpha_{algo}$. $\Delta\alpha_{t+1}$ is defined as follows and can be rewritten in terms of the violation rate:
\begin{align}
\begin{split}
    \Delta\alpha_{t+1} &= \Delta\alpha_t + \eta(\text{err}_t - \alpha_{algo}) 
    =\Delta\alpha_1 + \eta\cdot(\sum\limits^t_{t'=1}\text{err}_{t'}-\alpha_{algo}\cdot t)\\
    &=\Delta\alpha_1 + \eta\cdot t \cdot(\text{violation-rate}(t) - \alpha_{algo}),
\end{split}
\end{align}
where $\eta$ is an update rate hyperparameter; $\alpha_{algo}$ is the per-step violation rate; and $\alpha$ is the user-defined violation rate. $\alpha_{algo}$ is defined as
\begin{equation}\label{eq:alpha-algo}
    \alpha_{algo} = \frac{1}{T-1}(T\alpha - 1 - \frac{1}{\eta} + \frac{\Delta\alpha_1}{\eta}).
\end{equation}
Finally, $\varphi(\Delta\alpha_t)$ from Eq.~\ref{eq:beta_t1} can be defined in a variety of ways, though the following method used in \citeauthor{zhang2024} \citeyearpar{zhang2024} allows for fast adaptation of $\beta$:
\begin{equation}
\varphi(\Delta\alpha_t) = F^{-1}((\text{clip}(\Delta\alpha_t) + 1)/2),
\end{equation}
where $\text{clip}(\cdot)$ is a function that clips the endpoints of the input between 0 and 1, and $F^{-1}$ is the inverse cumulative distribution function of a standard Gaussian. $\beta_t$ increases slowly for low values of $\Delta\alpha_t$ and quickly as $\Delta\alpha_t$ approaches 1. See Fig. 3 of \citeauthor{zhang2024} \citeyearpar{zhang2024} for an illustration and \citeauthor{feldman2023} \citeyearpar{feldman2023} for a discussion of adaptation functions. 
\begin{algorithm}
  \caption{SafeOpt with Counterfactual Policy Constraints (SafeOpt-CPC)}
  \label{alg:safe-bocp}
  \begin{algorithmic}
    \STATE {\bfseries Input:} data $\mathcal{D}$, prior models for $f$ and $q$, initial safe set $\mathcal{S}_0$, initial decision $\mathbf{x}_0$, total iterations $T$, total observation iterations $T_0$, target rate $\alpha$, target rate $\epsilon$, update rate $\eta > 0$
    \STATE Split data $\mathcal{D}$ into $\mathcal{D}_{train}$ and $\mathcal{D}_{cal}$
    \STATE Train estimator $\hat{Y}$ on $\mathcal{D}_{train}$
    \IF{there may be a covariate shift between the calibration and test data}
        \STATE Observe test data for $T_0$ timesteps
    \ENDIF
    \FOR{$i=0 \text{ to } length(\mathcal{D}_{cal})$}
        \STATE Compute weights $w_i$ according to Eq.~\ref{eq:w_i} or Eq.~\ref{eq:w_i-covariate-shift} 
    \ENDFOR
    \STATE Calculate $\sum_{j=1}^n w_j$
    \STATE Calculate the scores on $\mathcal{D}_{cal}$ as in Eq.~\ref{eq:counterfactual-score} or Eq.~\ref{eq:score-lower}
    \STATE Calculate the violation rate $\alpha'$ according to Eq.~\ref{eq:alpha-prime}
    \FOR{$t=0$ {\bfseries to} $T$}
        \STATE Observe $f(\mathbf{x}_t)$ and $Y(\mathbf{x}_t)$
        \IF{$\mathbf{x}_t$ is the same as the standard(s) of care or \\$\mathbf{x}_t$ is the decision under some safe fallback policy (usually $\pi_{soc}$)} 
            \STATE $q(\mathbf{x_t}) = \omega$ 
        \ELSE
            \STATE Compute weight $w_{t}$ for current $\mathbf{x}_t$
            \STATE Compute the normalized weights $\Tilde{w}_i$ and $\Tilde{w}_{t}$ as in Eq.~\ref{eq:norm_w}
            \STATE Form the conformal upper or lower bounds according to Eq.~\ref{eq:upper-bound-counterfactual} or Eq.~\ref{eq:lower-bound-counterfactual}
            \STATE Form $q(\mathbf{x_t})$ according to Eq.~\ref{eq:q},~\ref{eq:q-lower}, or~\ref{eq:q-max}
        \ENDIF
        \STATE $\mathbf{x}_{t+1} \leftarrow$ \textsc{SafeOpt-Helper}() \COMMENT{Continue online conformal SafeOpt (Alg.~\ref{alg:safeopt-helper})}
    \ENDFOR
  \end{algorithmic}
\end{algorithm}

\section{Methodology}\label{sec:methodology}
Rather than creating constraints using the observed values of an outcome (or post-treatment covariate), we may wish to create constraints using counterfactuals. Specifically, we may want to ensure that the output of Bayesian optimization is not too much worse than a safe, standard-of-care policy $\pi_{soc}$ at every timestep $t$. In our setting, $\mathbf{x}$ is observed after every timestep. However, the outcome of the standard-of-care policy is not observed since that policy was not applied. 

To perform valid conformal prediction, we need to specify a probability of failure $\epsilon$ in addition to $\alpha$. For example, we might want to ensure that a constraint of the form $Y(\mathbf{x}_{t}) - Y(\mathbf{x}^{soc}_t)$ is at least $-\omega$ for 95\% of the samples ($\alpha=0.05$). However, if $\epsilon=0.01$, then there is a $1\%$ chance of an uncaught constraint violation. Therefore, we need to adjust $\alpha$ to be slightly lower to account for this potential violations. Suppose that $\mathbf{x}_t$ is the output of Bayesian optimization at time $t$ and $\mathbf{x}^{soc}_t$ is chosen by the standard of care policy $\pi_{soc}$.

\subsection{SafeOpt with counterfactual policy constraints (SafeOpt-CPC)}\label{sec:cpc}
$Y(\mathbf{x}_{t}) - Y(\mathbf{x}^{soc}_t) + \omega$ is not measurable due to the counterfactual $Y(\mathbf{x}^{soc}_t)$, so we need to replace $Y(\mathbf{x}^{soc}_t)$ with an estimate when it is not observed. To maintain our safety guarantee (Eq.~\ref{eq:violation-rate}), we cannot simply plug in a point prediction for $Y(\mathbf{x}^{soc}_t)$ without risking constraint violations. Rather, we will create a conformal prediction for the counterfactual.

When $\mathbf{x}^{soc}_t$ is chosen by Bayesian optimization, the safety constraint is met trivially:  
\begin{equation}\label{eq:q}
q(\mathbf{x}_t) =
\begin{cases} 
 Y(\mathbf{x}_t) -  \hat{U}^{soc}_{t}(\mathbf{x}^{soc}_t) + \omega & \text{if } \mathbf{x}_t \neq \mathbf{x}^{soc}_t \\
 \omega & \text{if } \mathbf{x}_t = \mathbf{x}^{soc}_t
\end{cases}
\end{equation}
where 
\begin{equation}\label{eq:ideal-conformal-guarantee}
    \mathbb{P}(Y(\mathbf{x}^{soc}_t) \leq \hat{U}^{soc}_{t}(\mathbf{x}^{soc}_t)) \ge 1 - \epsilon.
\end{equation}
We use the upper bound $\hat{U}^{soc}_{t}(\mathbf{x}^{soc}_t)$ in Eq.~\ref{eq:q} since in the worst case, the outcome under the standard of care is higher than the actual outcome. Because the outcome under the counterfactual policy is not observed at every step, we cannot apply online conformal prediction based just on the error signal as in \citeauthor{gibbs2021} \citeyearpar{gibbs2021} and \citeauthor{feldman2023} \citeyearpar{feldman2023}. Rather, we will use offline, split conformal prediction to find $\hat{U}^{soc}_{t}(\mathbf{x}^{soc}_t)$.

In split conformal prediction, there is a training set $\mathcal{D}_{train}$ that is used to train a point predictor and a calibration set $\mathcal{D}_{cal}$ of size $n$ that is used to create quantiles. The point prediction and quantiles can then be used to create an interval estimate for a new data point. In our setting, we will have some pre-collected data that we will split into a training set $\mathcal{D}_{train}$ and calibration set $\mathcal{D}_{cal}$, and we will observe our test data sequentially. 

To create prediction intervals, we take the quantile of the \textit{scores} for the calibration set and our new test point. Scores measure how unusual our new test point is compared to our training data. Because we are using an upper one-sided conformal prediction, we define the score $S$ (Eq.~\ref{eq:counterfactual-score}) as the difference between the true $Y(\mathbf{x}))$ and the predicted $\hat{Y}(\mathbf{x})$ without taking the absolute value. $\hat{Y}(\mathbf{x})$ is a prediction from the point predictor trained on $\mathcal{D}_{train}$. 
\begin{equation}\label{eq:counterfactual-score}
    S(\mathbf{x}, Y(\mathbf{x})) =  Y(\mathbf{x}) - \hat{Y}(\mathbf{x}).
\end{equation}
The scores for the calibration data are expected to be exchangeable with the score of the new test point we are predicting. If they are not exchangeable, we can use weighting to restore exchangeability. In our case the training and calibration data will ideally have been collected by our standard-of-care policy since we are predicting the outcome under the standard of care. If not, we will use weighting. 

$\hat{U}^{soc}_{t}(\mathbf{x}^{soc}_t)$ can be defined as
\begin{equation}\label{eq:upper-bound-counterfactual}
\begin{split}
    &\hat{U}^{soc}_{t}(\mathbf{x}^{soc}_t) = \hat{Y}(\mathbf{x}^{soc}_t) + Q^{soc}(1-\epsilon; \sum\limits_{i=1}^n \Tilde{w}_i \delta_{S_i^{soc}} + \Tilde{w}_{t} \delta_{+\infty})
\end{split}
\end{equation}
where $\delta_S$ is the point mass at score $S$; $\delta_{+\infty}$ is a point mass at a high score corresponding to the new data point at timestep $t$ of SafeOpt; $Q^{soc}$ denotes the quantile; and $\Tilde{w}$ are normalized weights \citep{tibshirani2019}:
\begin{equation}\label{eq:norm_w}
    \Tilde{w}_i = w_i/(\sum\limits_{j=1}^n w_j + w_{t}) \text{ and } \Tilde{w}_{t} = w_{t}/(\sum\limits_{j=1}^n w_j + w_{t}).
\end{equation}
Note that if the calibration data was collected by the standard-of-care policy, and there is no distribution shift between the calibration set and the new test point, then the weights in Eq.~\ref{eq:norm_w} are all set to 1. This is equivalent to unweighted conformal prediction.

If our calibration data did not come from the standard-of-care policy (Scenario 1, Sec.~\ref{sec:example}), we will use inverse propensity scoring. Let $\pi_{cal}(\mathbf{x}_i)$ be the probability of $\mathbf{x}_i$ under the policy that collected the calibration data. Then, propensity scoring uses weights of the form: 
\begin{equation}\label{eq:w_i}
   w_i = \pi_{soc}( \mathbf{x}_i)/\pi_{cal}( \mathbf{x}_i)
\end{equation}
Sometimes, there is a covariate shift between the calibration and test data, but no labeled data can be collected to create a new calibration set (Scenario 2, Sec.~\ref{sec:example}). If unlabeled data can be collected, it can be leveraged to learn weights \citep{tibshirani2019}:
\begin{equation}\label{eq:w_i-covariate-shift}
   w_i =\mathbb{P}_{test}( \mathbf{x}_i) / \mathbb{P}_{cal}( \mathbf{x}_i),
\end{equation}
where $\mathbb{P}_{test}( \mathbf{x}_i)$ is the probability of $\mathbf{x}_i$ under the test distribution, and $\mathbb{P}_{cal}(\mathbf{x}_i)$ is the probability of $\mathbf{x}_i$ under the calibration distribution.

In Appx.~\ref{appx:nonstationarity} we discuss how non-stationarity affects the weights $w_i$. We specifically examine two common covariate shift scenarios (changepoints and drift) and offer principled algorithms for both. When a changepoint in the covariates is detected, we can pause SafeOpt and re-calculate the weights before resuming. If drift is detected, we can create time-decaying weights. To avoid the calibration set becoming out-of-date, we take a standard-of-care action every other timestep. See Appx.~\ref{appx:nonstationarity} for details, algorithms, and experiments. 

\subsubsection{Extensions to different $q(\mathbf{x})$ constraints}
Eq.~\ref{eq:q} assumes that higher values of $Y$ are better, so we want $Y(\mathbf{x}_t)$ to be higher than $Y(\mathbf{x}^{soc}_t)$. If we instead want to minimize $Y$, then we can reformulate the constraint as: 
\begin{equation}\label{eq:q-lower}
q(\mathbf{x}_t) =
\begin{cases} 
  \hat{L}^{soc}_{t}(\mathbf{x}_t) - Y(\mathbf{x}_t) + \omega & \text{if } \mathbf{x}_t \neq \mathbf{x}^{soc}_t \\
 \omega & \text{if } \mathbf{x}_t = \mathbf{x}^{soc}_t
\end{cases}
\end{equation}
where $\hat{L}^{soc}_t$ is defined as
\begin{equation}\label{eq:lower-bound-counterfactual}
\begin{split}
    &\hat{L}^{soc}_{t}(\mathbf{x}_t) = \hat{Y}(\mathbf{x}^{soc}_t) - Q^{soc}(1-\epsilon; \sum\limits_{i=1}^n \Tilde{w}_i \delta_{S_i^{soc}} + \Tilde{w}_{t} \delta_{+\infty})
\end{split}
\end{equation}
\begin{equation}\label{eq:conformal-guarantee-lower}
    \mathbb{P}(Y(\mathbf{x}^{soc}_t) \geq \hat{L}^{soc}_{t}(\mathbf{x}_t)) \ge 1 - \epsilon
\end{equation} 
The corresponding score function would then be 
\begin{equation}\label{eq:score-lower}
    S(\mathbf{x}, Y(\mathbf{x})) = \hat{Y}(\mathbf{x}) - Y(\mathbf{x}).
\end{equation}
Second, we may wish to define $q(\mathbf{x})$ with more than one estimated component. Suppose we would like $Y(\mathbf{x}_t)$ to be higher than the values of $Y$ under two different standards of care: $Y(\mathbf{x}_t) - \max(Y(\mathbf{x}_{soc1}), Y(\mathbf{x}_{soc2})) + \omega$. Then, we could define the constraint function as
\begin{equation}\label{eq:q-max}
q(\mathbf{x}_t) =
\begin{cases} 
 Y(\mathbf{x}_t) - \max(\hat{U}^{soc1}_{t}(\mathbf{x}_t), \hat{U}^{soc2}_{t}(\mathbf{x}_t)) + \omega & \text{if } \mathbf{x}_t \neq \mathbf{x}_{soc1} \neq \mathbf{x}_{soc2} \\
 \omega & \text{if } \mathbf{x}_t = \mathbf{x}_{soc1} = \mathbf{x}_{soc2}
\end{cases}
\end{equation}
Note that both $\hat{U}^{soc1}_{t}(\mathbf{x}_t)$ and $\hat{U}^{soc2}_{t}(\mathbf{x}_t)$ in Eq.~\ref{eq:q-max} have conformal guarantees of the form in Eq.~\ref{eq:ideal-conformal-guarantee}. We can combine them into one conformal guarantee using a union bound:
\begin{equation}\label{eq:conformal-guarantee-multiple}
    \mathbb{P}(Y(\mathbf{x}^{soc,1}_t) \leq \hat{U}^{soc,1}_{t}(\mathbf{x}_t) \text{ and } ... \text{ and } Y(\mathbf{x}^{soc,m}_t) \leq \hat{U}^{soc,m}_{t}(\mathbf{x}_t)) \ge 1 - \epsilon_1 - ... - \epsilon_m.
\end{equation}
In this case $m=2$. We can then combine $\epsilon_1 + ... + \epsilon_m = \epsilon'$. 
\subsubsection{Adapting $\alpha$}\label{sec:adapting-alpha}
Because we are nesting conformal predictions, we will need to adjust the target violation rate $\alpha$ since the violation rate will be slightly higher than expected due to values of $Y(\mathbf{x}^{soc,i}_t)$ that fall above $\hat{U}^{soc,i}$.  Therefore, we define a practical maximum violation rate $\alpha'$:
\begin{equation}\label{eq:alpha-prime}
    \alpha' + \epsilon(1-\alpha') \le \alpha \equiv \alpha' \le (\alpha - \epsilon)/(1 - \epsilon).
\end{equation}
After computing $\alpha'$, Eq.~\ref{eq:alpha-algo} can be rewritten as $\alpha_{algo} = \frac{1}{T-1}(T\alpha' - 1 - \frac{1}{\eta} + \frac{\Delta\alpha_1}{\eta}).$

As an example, if we want to ensure that the safety constraint is met for 95\% of samples and $\epsilon = 0.01$, then $(0.05 - 0.01)/(1-0.01) \approx 0.04$. In this case, we would target a violation rate of $\alpha'\approx0.04$ since the stricter target rate would account for the uncaught violations. 

Alg.~\ref{alg:safe-bocp} adapts conformal SafeOpt to use counterfactual policy constraints (SafeOpt-CPC). 

\subsection{Safety proof sketch}\label{sec:proof}
The method in Sec.~\ref{sec:cpc} assumes \textit{weighted exchangeability} (\citeauthor{tibshirani2019} \citeyearpar{tibshirani2019}, Definition 1), which means that we assume that our weights are correctly specified for the test set. Weighted exchangeability also implies that when non-stationarity is present, it can be immediately detected and the weights can be recalculated. If weighted exchangeability is violated, there is a coverage gap bounded by Eq.~\ref{eq:gap-not-dependent} or Eq.~\ref{eq:gap-data-dependent} in Appx.~\ref{appx:coverage-gap}. 

Thm.~\ref{thm:safety-constraint} states the safety constraint from Eq.~\ref{eq:violation-rate} is met by the proposed SafeOpt method laid out in Sec.~\ref{sec:methodology}. Because we used conformal prediction, we can leverage $\epsilon$ to guarantee safety under weighted exchangeability. See Appx.~\ref{appx:proof} for the full proof.
\begin{theorem}\label{thm:safety-constraint}
    Under weighted exchangeability, Alg.~\ref{alg:safe-bocp} satisfies the safety constraint $\mathbb{P}(q(\mathbf{x_t}) \ge 0) \ge 1-\alpha$ for a user-defined violation rate $\alpha \in (0,1]$. 
\end{theorem}
\textit{Proof sketch.} We perform split conformal prediction to estimate $Y_{soc}$. If the training and calibration sets come from the standard-of-care policy, then the calibration set is exchangeable with the new test point. If not, we can use weighted split conformal prediction, as described in Sec.~\ref{sec:cpc}. Therefore, $\mathbb{P}(Y(\mathbf{x}^{soc}_t) \leq \hat{U}^{soc}_{t}(\mathbf{x}^{soc}_t)) \ge 1 - \epsilon$ or $\mathbb{P}(Y(\mathbf{x}^{soc}_t) \ge \hat{L}^{soc}_{t}(\mathbf{x}^{soc}_t)) \ge 1 - \epsilon$ hold for each conformal estimate. If there is more than one conformal estimate, we can use a union bound to create one combined conformal guarantee, as in Eq.~\ref{eq:conformal-guarantee-multiple}. In particular, we can combine the $\epsilon$ values to create $\epsilon_1 + ... + \epsilon_m = \epsilon'$.

SafeOpt with online conformal prediction stays below the violation rate $\alpha$, as proven by Thm. 2 of \citeauthor{zhang2024} \citeyearpar{zhang2024}. 

For a moment, let's suppose that each $Y(\mathbf{x}^{soc,i}_t)$ is observed. Therefore, $\mathbb{P}(q(\mathbf{x}_t) < 0 )) < \alpha'$ because of online conformal prediction. Because each $Y(\mathbf{x}^{soc,i}_t)$ is observed, each $Y(\mathbf{x}^{soc,i}_t)$ and our estimate of $Y(\mathbf{x}^{soc,i}_t)$ are the same, so the probability of a miscoverage event from offline prediction is 0. The probability of event $\mathbb{P}(q(\mathbf{x}_t) < 0 )) < \alpha'$ or an offline miscoverage event occurring is $\alpha' + 0(1-\alpha')$.

Suppose we replace each $Y(\mathbf{x}^{soc,i}_t)$ with a conformal prediction with coverage $1-\epsilon'$. We can replace $\alpha' + 0(1-\alpha')$ with $\alpha' + \epsilon'(1-\alpha')$, which is less than or equal to $\alpha$. Therefore, $\alpha$ includes all violations associated with the online conformal guarantee ($\alpha'$) plus all miscoverage events associated with the offline guarantee ($\epsilon'$) minus miscoverage events that have already been covered by the online guarantee ($\alpha'\epsilon'$). $\qed$

The purpose of Algs.~\ref{alg:changepoint} and~\ref{alg:drift} is to calculate better weights under non-stationarity. Because Thm.~\ref{thm:safety-constraint} assumes weighted exchangeability, Algs.~\ref{alg:changepoint} and~\ref{alg:drift} also satisfy the safety constraint. See Corollary~\ref{cor:safety-constraint-nonstationarity}.

\section{Experiments}\label{sec:experiments}

We evaluate the proposed method on data generated by a chemical reaction simulator \citep{kang2019} and the MovieLens dataset \citep{harper2015}. The advantage of the chemical reaction simulator is that we have full access to the counterfactuals. For MovieLens, we assume that the counterfactual rating is the same as the user's rating for the standard-of-care recommendation.

For both the chemical reaction simulator and the MovieLens dataset, we compare SafeOpt-CPC to Bayesian optimization without constraints, as well as to SafeOpt with full knowledge of the counterfactual (an oracle). We consider both the case where the calibration and training sets use the standard of care policy, as well as the case where the calibration and training sets use a different policy. That policy is random in the case of the simulator and is the data collection policy in the case of MovieLens. We examine constraints of the form in Eq.~\ref{eq:q} and Eq.~\ref{eq:q-max} for both datasets.

For the chemical reaction experiments, $f(\mathbf{x})$ is the percent yield, while $Y$ is the percent selectivity minus the percent yield (see \cite{kang2019}). The standard-of-care policy randomly recommends actions within the range Temperature$=[100.0,120.0]$ and pH$=[0.7,0.9]$. The second standard of care policy $\pi_{soc2}$ (when relevant) recommends Temperature$=110.0$ and pH=$0.8$.

For the MovieLens experiment, suppose we recommend movies to specific users chosen from a pool. Each movie has a popularity score of $Z$, and each user gives each movie a rating $Y$. 
We wish to recommend less popular movies while receiving sufficiently high ratings. Therefore, $f(\mathbf{x})=-Z(\mathbf{x})$, and the safety constraint $q(\mathbf{x})$ is of the form Eq.~\ref{eq:q} or Eq.~\ref{eq:q-max}. The standard-of-care policy randomly recommends one of the top five most popular movies in the training set. The second standard of care policy $\pi_{soc2}$ (when relevant) recommends the top movie. The pool of test users consists of the five users that (1) have rated all the movies possibly recommended by the standard of care policy and (2) have the most rated movies in common. The pool of movies to recommend consists of the movies these users have in common. This ensures that we have a valid rating for the counterfactual and that the popularity score does not influence which movies are available to recommend at each step. To add realism, the training and calibration sets consist of movies from the pool and users not from the pool, so there is a covariate shift between calibration and testing. We re-weight the calibration set to account for this shift. 

Finally, we randomly generate synthetic datasets, as described in Appx.~\ref{appx:synthetic-data}. The purpose of the synthetic data is to generate a variety of $f(\mathbf{x})$ and $q(\mathbf{x})$ functions in case there are any idiosyncrasies in the MovieLens or chemical reaction settings. In short we generate two variables $Z$ and $Y$ using different Gaussian processes. $Y$ is affected by a context variable. $f(\mathbf{x}) = Z(\mathbf{x})$ and $q(\mathbf{x})$ is defined in Eq.~\ref{eq:q}. We run a sensitivity analysis examining the effects of the quality of $\hat{Y}(\mathbf{x}_t)$, the amount of noise, and weight misspecification on the results. 

See Appx.~\ref{appx:experimental-details} for further experimental details and Appx.~\ref{appx:nonstationarity} for experiments in nonstationary settings. 

\section{Results}\label{sec:results}
As shown in the bar plots in Fig.~\ref{fig:main-results}, SafeOpt-CPC identifies potential violations below the rate $\alpha$. Of these potential violations caught by the algorithm, only a few end up being true violations. The oracle baseline gets the true violation rate closer to the rate $\alpha$ since it is not as conservative. Standard Bayesian optimization with no safety constraints has a potential and true violation rate far above the allowed violation rate $\alpha$. 

The $\omega$ tolerance graphs in Fig.~\ref{fig:main-results} show that as the tolerance term increases, SafeOpt-CPC falls back to the standard-of-care less often. $\omega$ is the tolerance term in Eq.~\ref{eq:q} and~\ref{eq:q-max}. Note that for the MovieLens experiments, we can see some sharp drops near the boundary of each star rating. The maximum $f(\mathbf{x})$ value found increases as the tolerance term increases. 

Fig.~\ref{fig:main-results} and Appx.~\ref{appx:additional-movielens} show that these trends are the same when the training and calibration set are collected from a different policy. The top row of Fig.~\ref{fig:main-results} shows the results when the training and calibration sets are collected from the standard-of-care policy. The second row shows the results when the sets are from the data collection policy. In both cases, $q(\mathbf{x})$ is defined as in Eq.~\ref{eq:q}. 

Fig.~\ref{fig:main-results} and Appx.~\ref{appx:additional-reaction} show that the above trends hold for the constraints defined in Eq.~\ref{eq:q} and Eq.~\ref{eq:q-max}. The third row of Fig.~\ref{fig:main-results} corresponds to Eq.~\ref{eq:q}, and the fourth row corresponds to Eq.~\ref{eq:q-max}. 

Overall, these results support the claim that SafeOpt-CPC stays below the allowed violation rate under different data collection policies for training and calibration and under different numbers of estimated components in $q(\mathbf{x})$. 

In Appx.~\ref{appx:sensitivity-analysis} we perform a sensitivity analysis and show that the quantiles in Eq.~\ref{eq:upper-bound-counterfactual} become larger under noise in the data and under high-bias and high-variance estimates of $Y$. We also show that the coverage gap increases with the weights' inaccuracy for covariate shift and non-stationarity. Finally, we show that coverage holds under non-stationarity with well-specified weights. 

\begin{figure}
\centering
\includegraphics[width=\columnwidth]{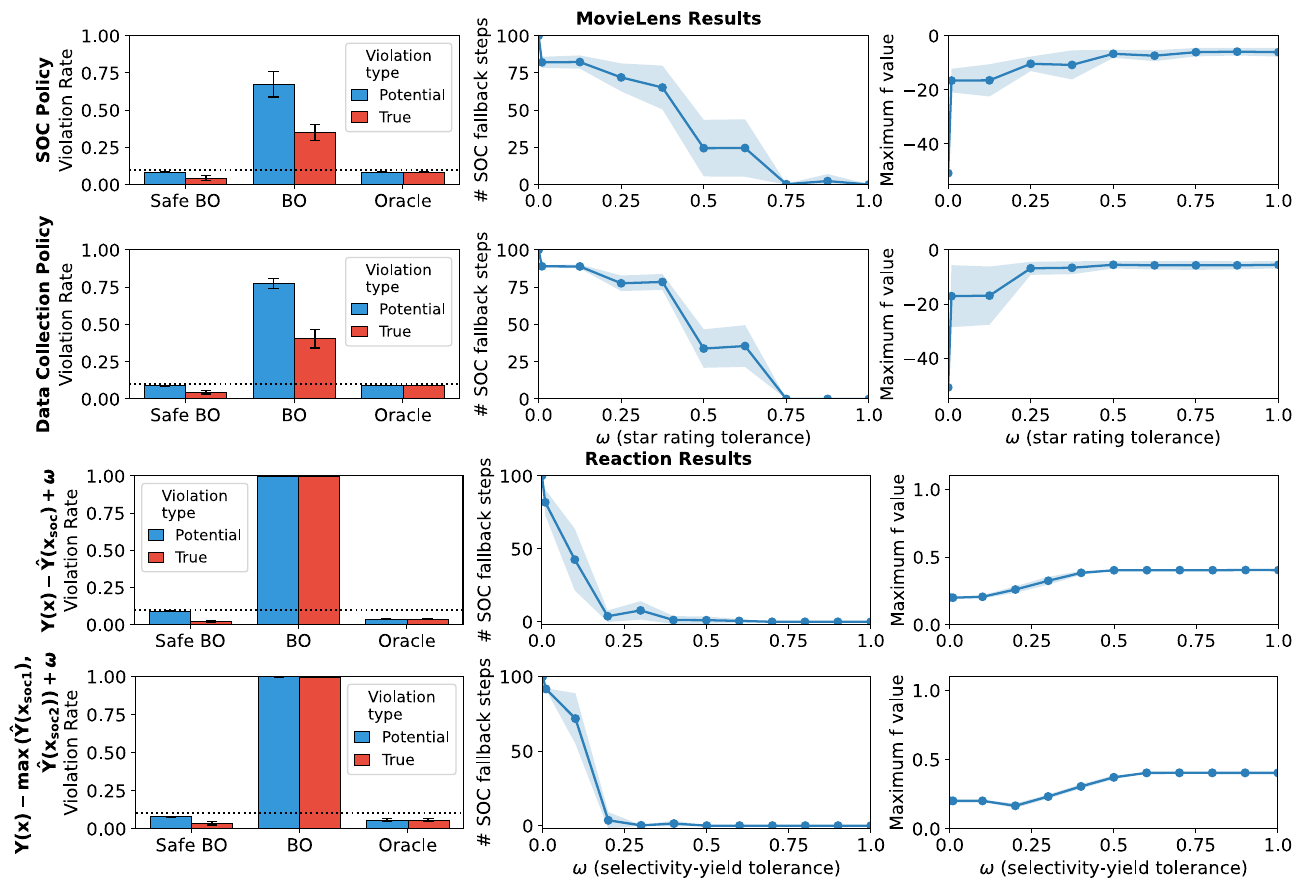} 
\caption{\textbf{MovieLens and chemical reaction results.} (Left) The number of violations and perceived violations for the oracle, standard BO, and SafeOpt-CPC ($\omega=0.01$). (Middle) The number of timesteps that SafeOpt-CPC reverts to the standard-of-care policy for different values of $\omega$ (tolerance levels). (Right) The maximum safe $f(\mathbf{x})$ value for different values of $\omega$. Note that the ratings for the MovieLens results are normalized between 0 and 1, so the range of $\omega$ is also [0,1]. 
}
\label{fig:main-results}
\end{figure}


\section{Discussion and conclusion}\label{sec:discussion}
In our results, we found that SafeOpt-CPC consistently keeps violations below the specified rate, despite the estimated components in the constraint. We also find that the number of potential or perceived violations is generally higher than the number of true violations. As expected, the number of standard of care fallback steps decreases as the tolerance increases, and the optimized value of $f(\mathbf{x})$ increases as the tolerance increases. 

Recall that this work focuses on safety constraints that include estimates of counterfactuals. That said, the split conformal prediction method laid out in Sec.~\ref{sec:methodology} can be trivially extended to any constraint that requires an estimated quantity, rather than an observed one. 

A limitation of our work is that if the quantiles are too large and/or the tolerance term is very low, the method may revert back to a given standard of care policy, which can be overly conservative. Further, if there are too many estimated quantities in the constraint, then $\alpha'$ can become very low, allowing for very few violations. For example, if $\alpha=0.05$ and there are four estimated quantities each at $\epsilon=0.01$, then $\alpha'$ decreases to only $0.01$.


This work targets applications with strict safety constraints, which is important for high-risk applications of machine learning. We would like to emphasize that controlling the violation rate alone is not sufficient for most high-risk applications, such as medical applications. Though we discuss only the violation rate, this work can be easily extended to other constraint functions. \citeauthor{feldman2023} \citeyearpar{feldman2023} discusses constraints functions other than violation rate, as well as how to control for multiple constraints simultaneously. We leave the extension of these methods to Bayesian optimization to future work. 

\bibliographystyle{plainnat}
\bibliography{neurips_2026}

@article{berkenkamp2023,
  title={Bayesian optimization with safety constraints: safe and automatic parameter tuning in robotics},
  author={Berkenkamp, Felix and Krause, Andreas and Schoellig, Angela P},
  journal={Machine Learning},
  volume={112},
  number={10},
  pages={3713--3747},
  year={2023},
  publisher={Springer}
}

@article{feldman2023,
title = "Achieving Risk Control in Online Learning Settings",
author = "Shai Feldman and Liran Ringel and Stephen Bates and Yaniv Romano",
journal = {Transactions on Machine Learning Research},
year = "2023",
month = nov,
day = "1",
volume = "2023",
}

@inproceedings{sui2015,
  title={Safe exploration for optimization with Gaussian processes},
  author={Sui, Yanan and Gotovos, Alkis and Burdick, Joel and Krause, Andreas},
  booktitle={International Conference on Machine Learning},
  pages={997--1005},
  year={2015},
  organization={PMLR}
}

@article{lei2021,
  title={Conformal inference of counterfactuals and individual treatment effects},
  author={Lei, Lihua and Cand{\`e}s, Emmanuel J},
  journal={Journal of the Royal Statistical Society Series B: Statistical Methodology},
  volume={83},
  number={5},
  pages={911--938},
  year={2021},
  publisher={Oxford University Press}
}

@article{gibbs2021,
  title={Adaptive conformal inference under distribution shift},
  author={Gibbs, Isaac and Candes, Emmanuel},
  journal={Advances in Neural Information Processing Systems},
  volume={34},
  pages={1660--1672},
  year={2021}
}

@article{zhang2024,
  title={Bayesian optimization with formal safety guarantees via online conformal prediction},
  author={Zhang, Yunchuan and Park, Sangwoo and Simeone, Osvaldo},
  journal={IEEE Journal of Selected Topics in Signal Processing},
  volume={19},
  number={1},
  pages={45--59},
  year={2024},
  publisher={IEEE}
}

@article{barber2023,
  title={Conformal prediction beyond exchangeability},
  author={Barber, Rina Foygel and Candes, Emmanuel J and Ramdas, Aaditya and Tibshirani, Ryan J},
  journal={The Annals of Statistics},
  volume={51},
  number={2},
  pages={816--845},
  year={2023},
  publisher={Institute of Mathematical Statistics}
}

@article{tibshirani2019,
  title={Conformal prediction under covariate shift},
  author={Tibshirani, Ryan J and Foygel Barber, Rina and Candes, Emmanuel and Ramdas, Aaditya},
  journal={Advances in neural information processing systems},
  volume={32},
  year={2019}
}

@article{harper2015,
author = {Harper, F. Maxwell and Konstan, Joseph A.},
title = {The MovieLens Datasets: History and Context},
year = {2015},
issue_date = {January 2016},
publisher = {Association for Computing Machinery},
address = {New York, NY, USA},
volume = {5},
number = {4},
issn = {2160-6455},
journal = {ACM Trans. Interact. Intell. Syst.},
month = dec,
articleno = {19},
numpages = {19}
}

@article{kang2019,
  title={Glucose to 5-hydroxymethylfurfural: origin of site-selectivity resolved by machine learning based reaction sampling},
  author={Kang, Pei-Lin and Shang, Cheng and Liu, Zhi-Pan},
  journal={Journal of the American Chemical Society},
  volume={141},
  number={51},
  pages={20525--20536},
  year={2019},
  publisher={ACS Publications}
}

@inproceedings{bergmann2020,
  title={Safe Bayesian optimization under unknown constraints},
  author={Bergmann, Daniel and Graichen, Knut},
  booktitle={2020 59th IEEE Conference on Decision and Control (CDC)},
  pages={3592--3597},
  year={2020},
  organization={IEEE}
}

@article{letham2019,
  title={Constrained Bayesian optimization with noisy experiments},
  author={Letham, Benjamin and Karrer, Brian and Ottoni, Guilherme and Bakshy, Eytan},
  year={2019}
}

@article{gelbart2014,
  title={Bayesian optimization with unknown constraints},
  author={Gelbart, Michael A and Snoek, Jasper and Adams, Ryan P},
  journal={arXiv preprint arXiv:1403.5607},
  year={2014}
}

@inproceedings{wang2024,
  title={Constrained Bayesian optimization under partial observations: Balanced improvements and provable convergence},
  author={Wang, Shengbo and Li, Ke},
  booktitle={Proceedings of the AAAI Conference on Artificial Intelligence},
  volume={38},
  number={14},
  pages={15607--15615},
  year={2024}
}

@article{amani2019,
  title={Linear stochastic bandits under safety constraints},
  author={Amani, Sanae and Alizadeh, Mahnoosh and Thrampoulidis, Christos},
  journal={Advances in Neural Information Processing Systems},
  volume={32},
  year={2019}
}

@inproceedings{zhou2018,
  title={Budget-constrained multi-armed bandits with multiple plays},
  author={Zhou, Datong and Tomlin, Claire},
  booktitle={Proceedings of the AAAI Conference on Artificial Intelligence},
  volume={32},
  number={1},
  year={2018}
}

@article{moradipari2021,
  title={Safe linear thompson sampling with side information},
  author={Moradipari, Ahmadreza and Amani, Sanae and Alizadeh, Mahnoosh and Thrampoulidis, Christos},
  journal={IEEE Transactions on Signal Processing},
  volume={69},
  pages={3755--3767},
  year={2021},
  publisher={IEEE}
}

@InProceedings{pacchiano2021,
  title = 	 { Stochastic Bandits with Linear Constraints },
  author =       {Pacchiano, Aldo and Ghavamzadeh, Mohammad and Bartlett, Peter and Jiang, Heinrich},
  booktitle = 	 {Proceedings of The 24th International Conference on Artificial Intelligence and Statistics},
  pages = 	 {2827--2835},
  year = 	 {2021},
  editor = 	 {Banerjee, Arindam and Fukumizu, Kenji},
  volume = 	 {130},
  series = 	 {Proceedings of Machine Learning Research},
  month = 	 {13--15 Apr},
  publisher =    {PMLR},
}

@article{agrawal2016,
  title={Linear contextual bandits with knapsacks},
  author={Agrawal, Shipra and Devanur, Nikhil},
  journal={Advances in neural information processing systems},
  volume={29},
  year={2016}
}

@inproceedings{agrawal2014,
  title={Bandits with concave rewards and convex knapsacks},
  author={Agrawal, Shipra and Devanur, Nikhil R},
  booktitle={Proceedings of the fifteenth ACM conference on Economics and computation},
  pages={989--1006},
  year={2014}
}

@article{wu2015,
  title={Algorithms with logarithmic or sublinear regret for constrained contextual bandits},
  author={Wu, Huasen and Srikant, Rayadurgam and Liu, Xin and Jiang, Chong},
  journal={Advances in Neural Information Processing Systems},
  volume={28},
  year={2015}
}


\appendix

\begin{algorithm}
  \caption{SafeOpt-Helper}
  \label{alg:safeopt-helper}
  \begin{algorithmic}
    \STATE {\bfseries Input:} $\mathbf{x}_t$, $q(\mathbf{x}_t)$, $f(\mathbf{x}_t)$, $T$, $t$, $\alpha$, $\epsilon$, $\eta$, $\mathcal{O}_{t-1}$, $\Delta \alpha_t$, current models for $f$ and $q$, current safe set $\mathcal{S}$
    \STATE Calculate $\text{err}_t = \indicator(q(\mathbf{x_t}) < 0)$ 
    \IF{$t==0$}
        \STATE $\mathcal{O}_t = \{ \mathbf{x}_t, f(\mathbf{x}_t), q(\mathbf{x}_t) \}$
    \ELSE
        \STATE Update $\mathcal{O}_t = \mathcal{O}_{t-1} \cup \{ \mathbf{x}_t, f(\mathbf{x}_t), q(\mathbf{x}_t) \}$
    \ENDIF
    \STATE Update the models (e.g., Gaussian processes) for $f$ and $q$ with $\mathcal{O}_t$
    \STATE Update $\Delta \alpha_{t+1} = \Delta \alpha_t + \eta(\text{err}_t - \alpha_{algo})$
    \STATE Update $\beta_{t+1} = \varphi(\Delta \alpha_{t+1})$
    \STATE Create credible intervals $\mathcal{I}_f(\mathbf{x}|\mathcal{O}_t)$ and $\mathcal{I}_q(\mathbf{x}|\mathcal{O}_t)$
    \STATE Update safe set $\mathcal{S}_{t+1}$
    \STATE Create set of optimizers $\mathcal{M}_{t+1}$ and expanders $\mathcal{G}_{t+1}$
    \IF{$t==T$}
        \STATE $\mathbf{x}_{t+1} = \arg\max_{\mathbf{x}\in S_{t+1}} f_l(\mathbf{x}|\mathcal{O}_T)$
    \ELSE
        \STATE Form decision $\mathbf{x}_{t+1}$ as in Eq.~\ref{eq:x_t1}
    \ENDIF
  \end{algorithmic}
\end{algorithm}

\section{Proofs}\label{appx:proofs}
\subsection{Proof of Theorem~\ref{thm:safety-constraint}}\label{appx:proof}
For this proof, we assume \textit{weighted exchangeability}, defined in Definition 1 of \cite{tibshirani2019}. Note that if the calibration and test sets are exchangeable, then this assumption is not required. 

Second, we will need to prove that Eq.~\ref{eq:conformal-guarantee-multiple} has valid conformal coverage:
\begin{lemma}\label{lemma:union-bound}
    Given that we have $m$ different conformal guarantees of the form \\$\mathbb{P}(Y(\mathbf{x}^{soc,i}_t) \leq \hat{U}^{soc,i}_{t}(\mathbf{x}_t)) \ge 1 - \epsilon_i$, then \\$\mathbb{P}(Y(\mathbf{x}^{soc,1}_t) \leq \hat{U}^{soc,1}_{t}(\mathbf{x}_t) \text{ and } ... \text{ and } Y(\mathbf{x}^{soc,m}_t) \leq \hat{U}^{soc,m}_{t}(\mathbf{x}_t)) \ge 1 - \epsilon_1 - ... - \epsilon_m$ (Eq.~\ref{eq:conformal-guarantee-multiple}). 
\end{lemma}
We first rewrite the left hand side of our conformal guarantee:

$\mathbb{P}(Y(\mathbf{x}^{soc,1}_t) \le \hat{U}^{soc,1}_{t}(\mathbf{x}_t) \text{ and } ... \text{ and } Y(\mathbf{x}^{soc,m}_t) \leq \hat{U}^{soc,m}_{t}(\mathbf{x}_t)) \\= 1 - \mathbb{P}(Y(\mathbf{x}^{soc,1}_t) > \hat{U}^{soc,1}_{t}(\mathbf{x}_t) \text{ or } ... \text{ or } Y(\mathbf{x}^{soc,m}_t) > \hat{U}^{soc,m}_{t}(\mathbf{x}_t))$  

By the union bound, 

$\mathbb{P}(Y(\mathbf{x}^{soc,1}_t) > \hat{U}^{soc,1}_{t}(\mathbf{x}_t) \text{ or } ... \text{ or } Y(\mathbf{x}^{soc,m}_t) > \hat{U}^{soc,m}_{t}(\mathbf{x}_t)) \\ \le \mathbb{P}(Y(\mathbf{x}^{soc,1}_t) > \hat{U}^{soc,1}_{t}(\mathbf{x}_t)) + ... + \mathbb{P}(Y(\mathbf{x}^{soc,m}_t) > \hat{U}^{soc,m}_{t}(\mathbf{x}_t))$

By definition,

$\mathbb{P}(Y(\mathbf{x}^{soc,1}_t) > \hat{U}^{soc,1}_{t}(\mathbf{x}_t)) \leq \epsilon_1$, ..., $\mathbb{P}(Y(\mathbf{x}^{soc,m}_t) > \hat{U}^{soc,m}_{t}(\mathbf{x}_t)) \leq \epsilon_m$

Therefore,

$\mathbb{P}(Y(\mathbf{x}^{soc,1}_t) \le \hat{U}^{soc,1}_{t}(\mathbf{x}_t) \text{ and } ... \text{ and } Y(\mathbf{x}^{soc,m}_t) \leq \hat{U}^{soc,m}_{t}(\mathbf{x}_t)) \\ \ge 1 - (\epsilon_1 + ... + \epsilon_m) = 1 - \epsilon_1 - ... - \epsilon_m$\qed \\ \\

\textbf{Theorem~\ref{thm:safety-constraint}} \textit{Under weighted exchangeability, Alg.~\ref{alg:safe-bocp} satisfies the safety constraint $\mathbb{P}(q(\mathbf{x_t} \ge 0)) \ge 1-\alpha$ for a user-defined violation rate $\alpha \in (0,1]$. } \\

We are performing split conformal prediction to estimate $Y_{soc}$. If the training and calibration sets come from the standard-of-care policy, then the calibration set is exchangeable with the new test point. 

If not, we can use weighted split conformal prediction, as described in Sec.~\ref{sec:cpc}. See Sec. 2.2 of \citeauthor{tibshirani2019} \citeyearpar{tibshirani2019}, which shows that weighted split conformal prediction satisfies Eq.~\ref{eq:upper-bound-counterfactual} or Eq.~\ref{eq:lower-bound-counterfactual} if the weights are well-specified. \citeauthor{barber2023} \citeyearpar{barber2023}, Thm. 2 proves the coverage gap in Eq.~\ref{eq:gap-not-dependent} and extends Thm. 2 to data-dependent weights in Sec. 4.5 (Eq.~\ref{eq:gap-data-dependent}). 

Therefore, $\mathbb{P}(Y(\mathbf{x}^{soc}_t) \leq \hat{U}^{soc}_{t}(\mathbf{x}^{soc}_t)) \ge 1 - \epsilon$ or $\mathbb{P}(Y(\mathbf{x}^{soc}_t) \ge \hat{L}^{soc}_{t}(\mathbf{x}^{soc}_t)) \ge 1 - \epsilon$ hold for each conformal estimate. If there is more than one conformal estimate, we can use a union bound to create one combined conformal guarantee, as proven in Lemma~\ref{lemma:union-bound}. In particular, we can combine the $\epsilon$ values to create $\epsilon_1 + ... + \epsilon_m = \epsilon'$.

SafeOpt with online conformal prediction stays below the violation rate $\alpha$, as proven by Thm. 2 of \citeauthor{zhang2024} \citeyearpar{zhang2024}. 

Suppose that each $Y(\mathbf{x}^{soc,i}_t)$ is observed. Therefore, $\mathbb{P}(q(\mathbf{x}_t) < 0 )) < \alpha'$ because of online conformal prediction. Because each $Y(\mathbf{x}^{soc,i}_t)$ is observed, each $Y(\mathbf{x}^{soc,i}_t)$ and our estimate of $Y(\mathbf{x}^{soc,i}_t)$ are the same, so the probability of a miscoverage event from offline prediction is 0. The probability of event $\mathbb{P}(q(\mathbf{x}_t) < 0 )) < \alpha'$ or an offline miscoverage event occurring is $\alpha' + 0(1-\alpha') = \alpha'$.

Suppose we replace each $Y(\mathbf{x}^{soc,i}_t)$ with a conformal prediction with coverage $1-\epsilon'$. We can replace $\alpha' + 0(1-\alpha')$ with $\alpha' + \epsilon'(1-\alpha')$. If a violation is detected by online conformal prediction, we do not need to count a second violation from miscoverage in offline conformal prediction, so we add $\epsilon'(1-\alpha')$ instead of $\epsilon'$ so as not to count the intersection between online and offline conformal violations twice. 

$\alpha' + \epsilon'(1-\alpha')$ is less than or equal to $\alpha$. Therefore, $\alpha$ includes all violations associated with the online conformal guarantee ($\alpha'$) plus all miscoverage events associated with the offline guarantee ($\epsilon'$) minus miscoverage events that have already been covered by the online guarantee ($\alpha'\epsilon'$). $\qed$ \\

The purpose of Algs.~\ref{alg:changepoint} and~\ref{alg:drift} is to calculate better weights under non-stationarity. Because Thm.~\ref{thm:safety-constraint} assumes weighted exchangeability, Algs.~\ref{alg:changepoint} and~\ref{alg:drift} also satisfy the safety constraint.
\begin{corollary}\label{cor:safety-constraint-nonstationarity}
    Under weighted exchangeability, Algs.~\ref{alg:changepoint} and~\ref{alg:drift} satisfy the safety constraint $\mathbb{P}(q(\mathbf{x_t} \ge 0)) \ge 1-\alpha$ for a user-defined violation rate $\alpha \in (0,1]$.  
\end{corollary}

\section{Coverage gaps and analysis}\label{appx:coverage-gap}
When the weighted exchangeability assumption is not met, we have a coverage gap: 
\begin{equation}
    P(Y(\mathbf{x}^{soc}_t) \leq \hat{U}^{soc}_{t}(\mathbf{x}^{soc}_t)) \ge 1 - \epsilon_t + gap_t
\end{equation}
The gap is measured in terms of total variation distance (TV). If the weights are not data-dependent, then the gap can be bounded by 
\begin{equation}\label{eq:gap-not-dependent}
    gap \le \sum\limits_{i=1}^n \Tilde{w_i} \cdot TV(R(\mathcal{D}_{cal+1}), R(\mathcal{D}_{cal+1}^i))
\end{equation}
\citep{barber2023}. $\mathcal{D}_{cal+1}$ is the calibration set with the new test point appended to the end: $(\mathcal{D}_{cal}, \mathbf{x}_t)$. $\mathcal{D}_{cal+1}^i$ is $\mathcal{D}_{cal+1}$ with point $i$ swapped for point $n+1$. $R(\mathcal{D}_{cal+1})$ is the vector of residuals for $\mathcal{D}_{cal+1}$. Note that the residuals and the score are not necessarily the same.  If the weights are data-dependent,
\begin{equation}\label{eq:gap-data-dependent}
    gap \le \mathbb{E}[ \sum\limits_{i=1}^n \Tilde{w_i} \cdot TV(R(\mathcal{D}_{cal+1}), R(\mathcal{D}_{cal+1}^i) | w_1,...,w_n) ]
\end{equation}
Data dependent weights are computed from the observed data, rather than being fixed in advance or known from the true distribution. The weight of one point can depend on the weights of the other points, so we condition on the weights and take the expectation \citep{barber2023}.\footnote{Because we are doing split conformal prediction, we ignore tags.} In Eq.~\ref{eq:w_i}, $P_{cal}$ is usually estimated from the sample, so the weights would be data-dependent. Alternatively, if a practitioner accounts for covariate drift by deciding the decay rate for the weights in advance, then the weights would not be data-dependent.

Note that we cannot calculate the coverage gap before running the algorithm because $Y$ is not observed until after an action is taken. Further, the gap is very expensive to calculate because we need to measure the TV distance for every new test point swapped with every point $i$ in the calibration set. In theory, you could estimate the gap based on the unlabeled data used to calculate the weights, though you would need to assume no shifts in $P(Y | \mathbf{x})$ or $P(\mathbf{x} | Y)$. 
\begin{equation}\label{eq:gap-guarantee}
\begin{split}
    &\epsilon_t' = \epsilon + gap_t \\
    &P(Y(\mathbf{x}^{soc}_t) \leq \hat{U}^{soc}_{t}(\mathbf{x}^{soc}_t)) \ge 1 - \epsilon_t'
\end{split}
\end{equation}
The overall chance that $Y(\mathbf{x}_{soc})$ falls outside of the interval between timesteps 1 and $T_0$ is therefore $ \epsilon' = (\sum_t^{T_0} \epsilon_t') / T_0$. 

\begin{algorithm}
  \caption{SafeOpt with counterfactual policy constraints (changepoint)}
  \label{alg:changepoint}
  \begin{algorithmic}
    \STATE {\bfseries Input:} data $\mathcal{D}$, prior models for $f$ and $q$, initial safe set $\mathcal{S}_0$, initial decision $\mathbf{x}_0$, total iterations $T$, observation iterations $T_0$, target overall violation rate $\alpha$, target counterfactual policy violation rate $\epsilon$, update rate $\eta > 0$
    \STATE Split data $\mathcal{D}$ into $\mathcal{D}_{train}$ and $\mathcal{D}_{cal}$
    \STATE $n \leftarrow length(\mathcal{D}_{cal})$
    \STATE Train estimator $\hat{Y}$ on $\mathcal{D}_{train}$
    \IF{there may be a covariate shift between the calibration and test data}
        \STATE Observe test data for $T_0$ timesteps
    \ENDIF
    \FOR{i=0 to n}
        \STATE Form weights $w_i$ according to Eq.~\ref{eq:w_i} or Eq.~\ref{eq:w_i-covariate-shift} 
    \ENDFOR
    \STATE Calculate $\sum_{j=1}^n w_j$
    \STATE Calculate the scores on $\mathcal{D}_{cal}$ as in Eq.~\ref{eq:counterfactual-score} or Eq.~\ref{eq:score-lower}
    \STATE Calculate violation rate $\alpha'$ according to Eq.~\ref{eq:alpha-prime}
    \FOR{$t=0$ {\bfseries to} $T$}
        \STATE Observe $f(\mathbf{x}_t)$ and $Y(\mathbf{x_t})$
        \IF{$\mathbf{x}_t$ is the same as the standard(s) of care \\or $\mathbf{x}_t$ is the decision under the safe fallback policy}
            \STATE $q(\mathbf{x_t}) = \omega$ 
        \ELSE
            \STATE Form weight $w_{t}$ for current $\mathbf{x}_t$
            \STATE Compute the normalized weight $\Tilde{w}_i$ as in Eq.~\ref{eq:norm_w}
            \STATE Compute the normalized weight $\Tilde{w}_{t}$ as in Eq.~\ref{eq:norm_w}
            \STATE Form the conformal upper or lower bounds according to Eq.~\ref{eq:upper-bound-counterfactual} or Eq.~\ref{eq:lower-bound-counterfactual}
            \STATE Form $q(\mathbf{x_t})$ according to Eq.~\ref{eq:q},~\ref{eq:q-lower}, or~\ref{eq:q-max}
        \ENDIF
        \IF{a changepoint is detected}
            \STATE Observe test data for $T_0$ timesteps
            \FOR{i=0 to n}
                \STATE Form weights $w_i$ according to Eq.~\ref{eq:w_i-covariate-shift}
            \ENDFOR
            \STATE Calculate $\sum_{j=1}^n w_j$
        \ENDIF
        \STATE $\mathbf{x}_{t+1} \leftarrow$ \sc{SafeOpt-Helper}() 
    \ENDFOR
  \end{algorithmic}
\end{algorithm}

\begin{algorithm}
  \caption{SafeOpt with counterfactual policy constraints (drift)}
  \label{alg:drift}
  \begin{algorithmic}
    \STATE {\bfseries Input:} data $\mathcal{D}$, prior models for $f$ and $q$, initial safe set $\mathcal{S}_0$, initial decision $\mathbf{x}_0$, total iterations $T$, observation iterations $T_0$, target overall violation rate $\alpha$, target counterfactual policy violation rate $\epsilon$, update rate $\eta > 0$
    \STATE Split data $\mathcal{D}$ into $\mathcal{D}_{train}$ and $\mathcal{D}_{cal}$
    \STATE $n \leftarrow length(\mathcal{D}_{cal})$
    \STATE Train estimator $\hat{Y}$ on $\mathcal{D}_{train}$
    \IF{there may be a covariate shift between the calibration and test data}
        \STATE Observe test data for $T_0$ timesteps
    \ENDIF
    \FOR{i=0 to n}
        \STATE Form time-decaying weights $w_i$ (for example, according to Eq.~\ref{eq:drift})
    \ENDFOR
    \STATE Calculate $\sum_{j=1}^n w_j$
    \STATE Calculate the scores on $\mathcal{D}_{cal}$ as in Eq.~\ref{eq:counterfactual-score} or Eq.~\ref{eq:score-lower}
    \STATE Calculate violation rate $\alpha'$ according to Eq.~\ref{eq:alpha-prime}
    \FOR{$t=0$ {\bfseries to} $T$}
        \IF{$t \% 2 == 0$}
            \STATE Observe $f(\mathbf{x}_t)$ and $Y(\mathbf{x_t})$
            \IF{$\mathbf{x}_t$ is the same as the standard(s) of care \\or $\mathbf{x}_t$ is the decision under the safe fallback policy}
                \STATE $q(\mathbf{x_t}) = \omega$ 
            \ELSE
                \STATE Form weight $w_{t}$ for current $\mathbf{x}_t$
                \STATE Compute the normalized weight $\Tilde{w}_i$ as in Eq.~\ref{eq:norm_w}
                \STATE Compute the normalized weight $\Tilde{w}_{t}$ as in Eq.~\ref{eq:norm_w}
                \STATE Form the conformal upper or lower bounds according to Eq.~\ref{eq:upper-bound-counterfactual} or Eq.~\ref{eq:lower-bound-counterfactual}
                \STATE Form $q(\mathbf{x_t})$ according to Eq.~\ref{eq:q},~\ref{eq:q-lower}, or~\ref{eq:q-max}
            \ENDIF
            \STATE $\mathbf{x}_{t+1} \leftarrow \pi_{soc}$
            \STATE Add $\mathbf{x}_{t+1}$ to $\mathcal{D}_{cal}$
        \ELSE
            \STATE $\mathbf{x}_{t+1} \leftarrow$ \sc{SafeOpt-Helper}()
        \ENDIF
    \ENDFOR
  \end{algorithmic}
\end{algorithm}

\section{Non-stationarity extension}\label{appx:nonstationarity}

SafeOpt with online conformal prediction \citep{zhang2024} does not assume stationarity, and we would like to preserve this property as much as possible for safety constraints with estimated components. Not accounting for non-stationarity when it is present will result in a larger coverage gap (see Sec.~\ref{appx:coverage-gap}). 

Unfortunately, we need to know or be able to detect the kind of non-stationarity in our application. Detection is often possible for non-stationarity in the covariates, so we limit our discussion to covariate shift. We outline two common scenarios, changepoints and drift, below and provide algorithms to account for each scenario. See \citeauthor{barber2023} \citeyearpar{barber2023} for a discussion of nonexchangeability, including nonexchangeability from non-stationarity, such as changepoints and drift.

\subsection{Motivating example}

\textbf{Scenario 3: Sudden shift.} We begin running SafeOpt at Hospital A, but we move to Hospital B halfway through the study. By observing the ages of the people in the study, we can re-weight the calibration data for the second half of the study. 

\textbf{Scenario 4: Gradual drift.} Our training and calibration sets were collected at Hospital A, and the patients recruited and scheduled for the study earlier were from Hospital A. The patients recruited and scheduled later tended to be from Hospital B. We look at when patients are scheduled to receive treatment via SafeOpt and calculate the rate of covariate drift. We then weight each point in the calibration set such that points collected further in the past are less important. To avoid the entire calibration set being far in the past by the end of the study, we alternate between running SafeOpt and running the standard-of-care policy. The standard-of-care step is added to our calibration set. 

\subsection{Algorithmic description}

If covariate shift is detected at a changepoint (Scenario 3), we can observe unlabeled data and recalculate the weights as in Eq.~\ref{eq:w_i-covariate-shift}. After the weights have been calculated, we can create conformal predictions for $Y_{soc}$ again and SafeOpt continues (Alg.~\ref{alg:changepoint}). 

If there is gradual covariate drift (Scenario 4), we can observe the rate of drift before starting Bayesian optimization so that the weights are set appropriately. Each point in the calibration set is weighted such that points collected further in the past are less important. For example, we could learn a monotone function $g$ such that 
\begin{equation}\label{eq:drift}
   w_i = \exp(-g(t-t_{i})) \cdot P_{test}( \mathbf{x}_i)/P_{cal}( \mathbf{x}_i)
\end{equation}
where $t_{i}$ is the time of data collection for point $i$, and $t$ is the current timestep. To ensure that the calibration set does not become out of date, we alternate between taking a step of Bayesian optimization and taking an action from $\pi_{soc}$. This $\pi_{soc}$ step is added to our calibration set (Alg.~\ref{alg:drift}).

\section{Additional experimental details}\label{appx:experimental-details}
Part of our code is based on \cite{berkenkamp2023}, which uses an MIT License. The MovieLens 100k dataset \citep{harper2015} is available at https://grouplens.org/datasets/movielens/100k/ with a special license listed here: https://files.grouplens.org/datasets/movielens/ml-100k-README.txt. The chemical reaction simulator \citep{kang2019} is available at https://github.com/VlachosGroup/Fructose-HMF-Model. It also has an MIT License. 

Each experiment was run on one CPU and used about 300 megabytes. Running 10 seeds took between 4 and 20 minutes, depending on the experiment. For example, experiments with no weighting ran in a few minutes, while experiments with different policies between calibration and test took about 20 minutes. 

For the MovieLens 100k dataset, we also wanted to include variation in the data, rather than running the algorithm on the same dataset every time. Therefore, we used bootstrap sampling on the user level to get a slightly different dataset for every random seed. 

The results for every dataset or simulator were run on 10 random seeds each. Ninety-five percent confidence intervals were computed by multiplying the standard error of the mean by a critical t-value from scipy.stats.t.ppf(). 

The hyperparameter $\eta$ was set to 0.2 for all experiments. In preliminary tests, we found that this was an appropriate value for 100 timesteps. We chose $\alpha=0.1$ since that value was used in many of the \cite{zhang2024} experiments, and it was a high enough value to accommodate a couple of estimated components. $\epsilon=0.01$ since error rates in conformal prediction are usually set to some small value, and $\epsilon$ needed to be small enough that $\alpha'$ was positive.  

Because the weights and scores are calculated and stored before the looping through each new test point, not much computational overhead is added compared to standard SafeOpt with online conformal prediction. To calculate the quantiles for each new test point, we need to iterate through at most $n+1$ data points, where $n$ is the size of the calibration set.  


All code can be found at [Github link redacted]. 

\subsection{MovieLens experiments}

We use matrix factorization on the training set to describe each movie as a set of movie features and each user as a set of user features. User-item pairs from the test set are not included when creating the factorization since they would not have been observed. Each movie is a set of 20 features, and each observed user is a set of 20 features. 

We initialize the safe set with one decision $\mathbf{x}_0$ from the standard-of-care policy and its corresponding observations of $f(\mathbf{x}_0)$ and $q(\mathbf{x}_0)$. 

In contextual Gaussian processes, the context is often appended to the decision. For the MovieLens experiments, the context is appended to each decision $\mathbf{x}$. The context includes the standard of care action, followed by the user features. Each decision is a set of item features. 

The kernels are factored linear kernels that treat the standard of care action context and the user context differently: 
\begin{equation}
\begin{split}
    &k(item, item', item_{soc}, item_{soc}', user, user') = \\& k(item, item') \cdot k(item_{soc}, item_{soc}') + k(item, item') \cdot k(user, user')
\end{split}
\end{equation}
where $k(a, a') = a^\top a'$.

\subsection{Chemical reaction experiments}

As in the MovieLens experiments, we initialize the safe set with a decision and observations from the standard-of-care policy. The context only includes the standard-of-care action in this case. The kernel is as follows: 
\begin{equation}
\begin{split}
    &k(item, item', item_{soc}, item_{soc}') = k(item, item') \cdot k(item_{soc}, item_{soc}') 
\end{split}
\end{equation}
where $k(a, a') = a^\top a'$.

\subsection{Synthetic experiments}\label{appx:synthetic-data}

To create the synthetic functions $Z$ and $Y$, we sample functions from two Gaussian processes. The process for $Z$ has an RBF kernel with lengthscale 1.0 and variance 1.0. The process for $Y$ has an RBF kernel variance 1.0. There is a single context variable $c$ that affects $Y$. The lengthscale is 0.1 for $item$, the choice that SafeOpt will make, and 1.0 for $c$. 
\begin{equation}
\begin{split}
    Z& \sim GP(0, RBF(1.0, 1.0)) \\
    Y& \sim GP(0, RBF(1.2, (1.0, 0.1))) + 5\cdot c
\end{split}
\end{equation}
We seed the Gaussian processes with support points between 0 and 1 for $item$ and 0 and 5 for $c$; sample a prior; and then solve for the weights that reconstruct the sampled GP function from the support points. See synthetic\_data.py for details. 

We initialize the safe set with a decision and observations from the standard-of-care. 

The kernel we use for $q$ is similar to the MovieLens kernel, except $c$ is our 1D context variable:
\begin{equation}
\begin{split}
    &k(item, item', item_{soc}, item_{soc}', c, c') = \\& k(item, item') \cdot k(item_{soc}, item_{soc}') + k(item, item') \cdot k(c, c')
\end{split}
\end{equation}
where $k(a, a') = a^\top a'$.

\section{Additional results}\label{appx:additional-results}

Note that $\alpha=0.1$ and $\epsilon=0.01$ for all plots in this section. We create confidence bounds based on 10 random seeds. 

\subsection{Additional MovieLens experiments}\label{appx:additional-movielens}
Fig.~\ref{fig:main-results} shows the results when the training and calibration data come from the standard-of-care policy or the data collection policy when $q(\mathbf{x})$ is defined as in Eq.~\ref{eq:q}. Fig.~\ref{fig:movielens-soc-results} shows the full results for both Eq.~\ref{eq:q} and Eq.~\ref{eq:q-max} under the standard-of-care policy, and Fig.~\ref{fig:movielens-obs-results} shows the full results for both Eq.~\ref{eq:q} and Eq.~\ref{eq:q-max} under the data collection policy. Note that some results are repeated from Fig.~\ref{fig:main-results}. The results do not differ much from the results in the main text.

\begin{figure}
\centering
\includegraphics[width=\columnwidth]{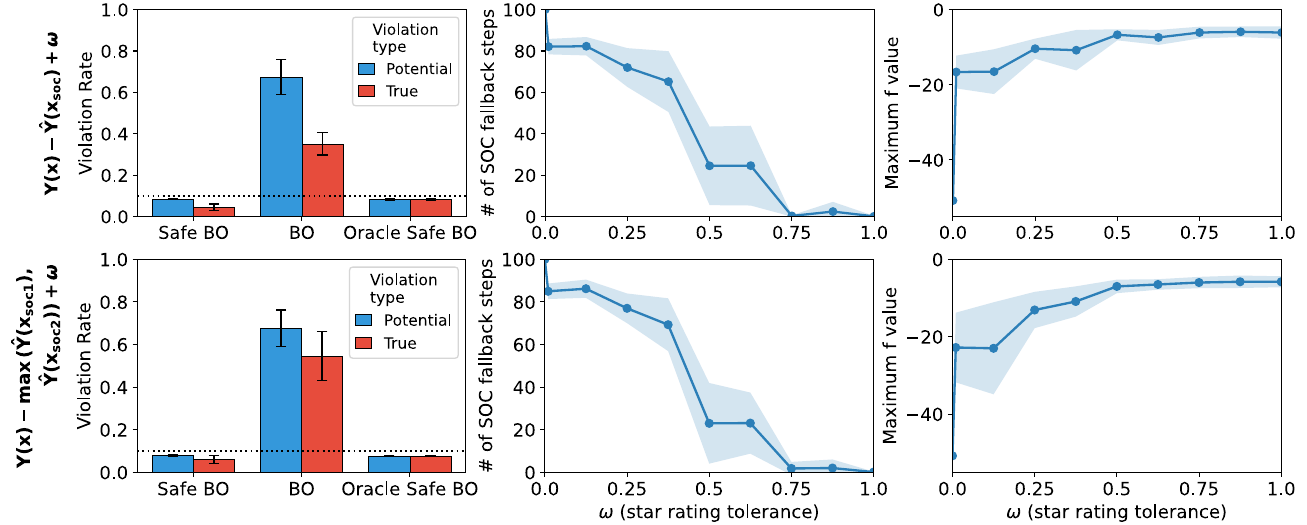} 
\caption{\textbf{MovieLens results with standard-of-care policy training and calibration.} 
}
\label{fig:movielens-soc-results}
\end{figure}
\begin{figure}
\centering
\includegraphics[width=\columnwidth]{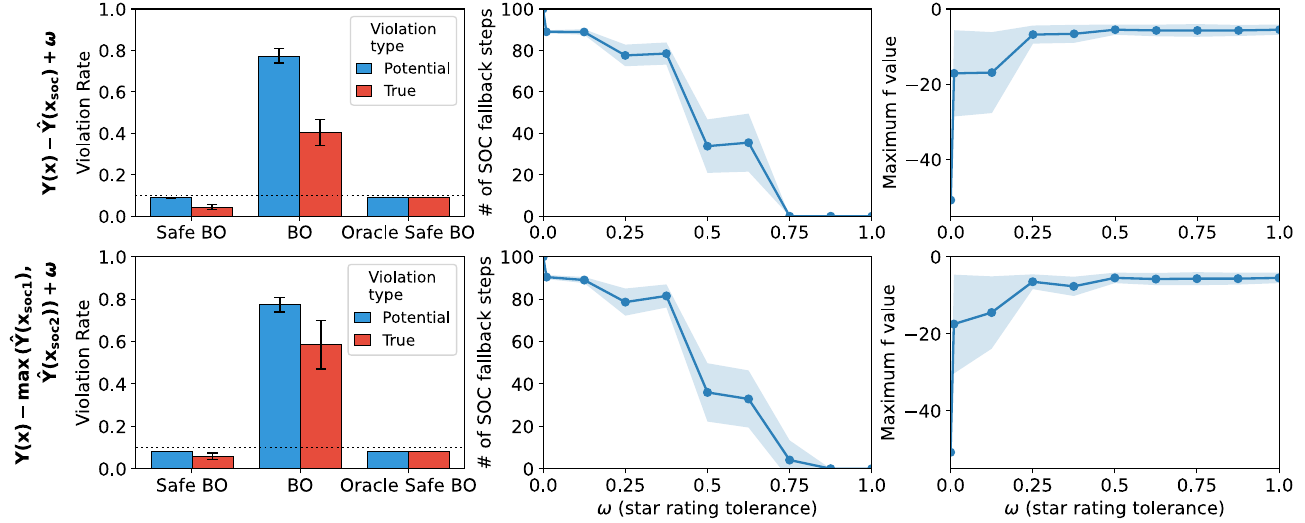} 
\caption{\textbf{MovieLens results with data collection policy training and calibration.} 
}
\label{fig:movielens-obs-results}
\end{figure}

\subsection{Additional chemical reaction experiments}\label{appx:additional-reaction}
Fig.~\ref{fig:main-results} shows the chemical reaction results under a random policy for training and calibration when $q(\mathbf{x})$ is defined as in Eq.~\ref{eq:q} and Eq.~\ref{eq:q-max}. Fig.~\ref{fig:reaction-soc-results} shows the results when the training and calibration data have been collected from the standard-of-care policy. The results do not differ much from the results in the main text.

\begin{figure}
\centering
\includegraphics[width=\columnwidth]{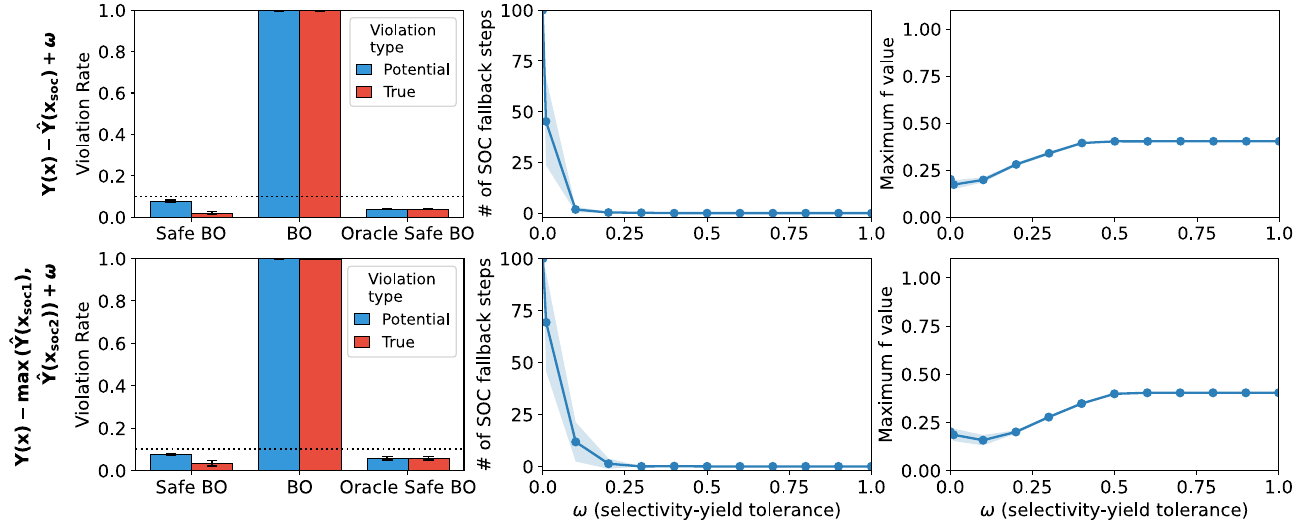} 
\caption{\textbf{Reaction results with standard-of-care policy training and calibration.} 
}
\label{fig:reaction-soc-results}
\end{figure}

\subsection{Synthetic data results}\label{appx:additional-synthetic}
We did not include the synthetic data experiments in the main text due to space constraints, but it is useful to include them in order to understand the later experiments in the non-stationarity experiments section and sensitivity analysis section. Because the non-stationarity experiments and sensitivity analysis use $q(\mathbf{x})$ from Eq.~\ref{eq:q} and use training and calibration data from the standard-of-care policy, Fig.~\ref{fig:synthetic-results} also uses these conditions. Note that for each seed, we randomly generate a new $Z$ and $Y$. 

We observe the same general trends in the synthetic results as we do in the MovieLens and reaction results. The primary difference is that the synthetic results quickly reach near-zero standard-of-care fallback steps and quickly reach the maximum $f$ value. 

\begin{figure}
\centering
\includegraphics[width=\columnwidth]{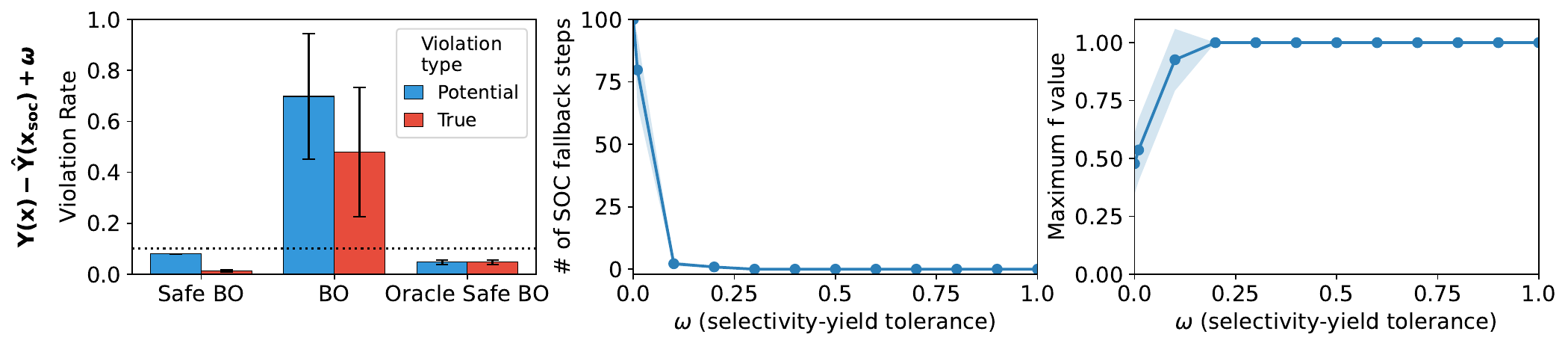} 
\caption{\textbf{Synthetic results.} 
}
\label{fig:synthetic-results}
\end{figure}

\subsection{Non-stationarity results}\label{appx:nonstationarity-results}
We add drift to the synthetic data's covariates so that the covariates transition from (0, 1) to (4, 5) at a rate proportional to $\gamma^{\Delta t}$, where $\gamma$ is a parameter plotted in Fig.~\ref{fig:drift}. 

The left-hand plot of Fig.~\ref{fig:drift} shows the violation rate when we correctly re-weight for drift according to $\gamma^{\Delta t}$. Recent observations are weighted so that they are more relevant, and past observations are less relevant. As described in Sec.~\ref{appx:nonstationarity}, we alternate between Bayesian optimization and taking a standard-of-care action for the calibration set. In the right-hand plot, we ignore drift altogether and run Alg.~\ref{alg:safe-bocp} as normal.

\begin{figure}
\centering
\includegraphics[width=\columnwidth]{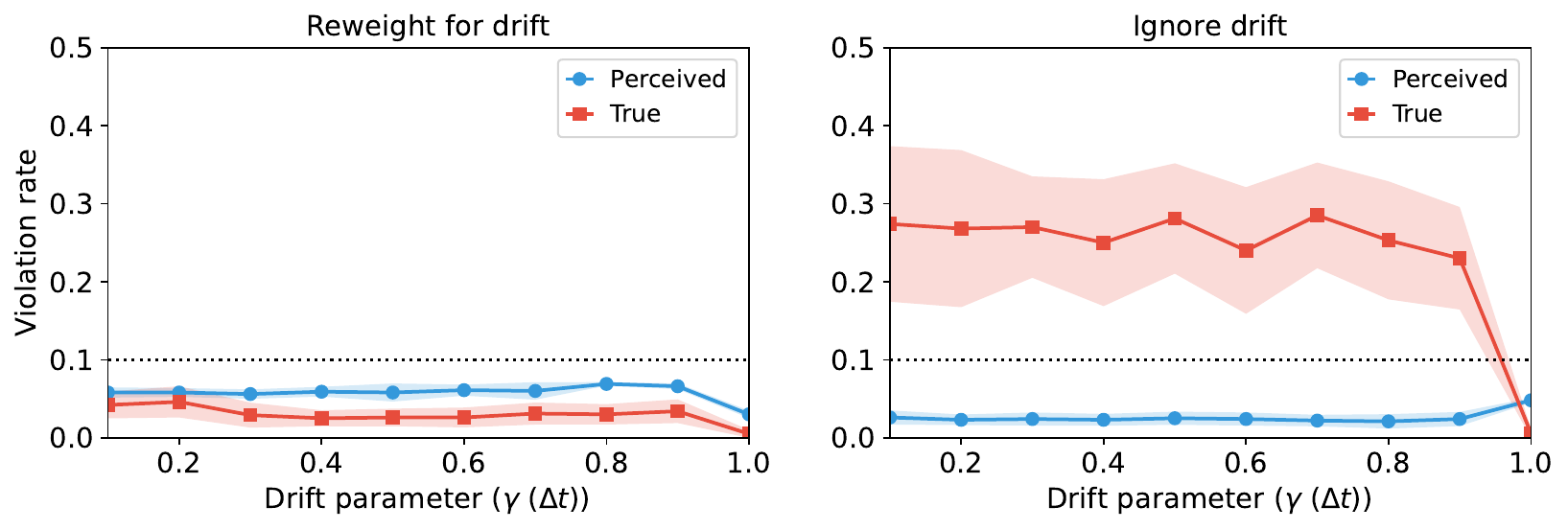} 
\caption{\textbf{Drift results.} 
}
\label{fig:drift}
\end{figure}

In Fig.~\ref{fig:changepoint} we examine the effects of changepoints. Before the changepoint, the covariates are between 0 and 1, while after the changepoint, they are between 4 and 5. We plot the violation rate versus the time of the changepoint. When we correctly re-weight for the changepoint, the violation rate stays below $\alpha$. The violation rate is slightly higher at lower timesteps because we only observe 100 points of new calibration data after the changepoint, so we have better estimates before the changepoint since we have more old calibration data.

In the right-hand plot, we ignore the changepoint. In the worst case, the changepoint occurs early. Because we are not re-weighting, conformal coverage does not hold. 
\begin{figure}
\centering
\includegraphics[width=\columnwidth]{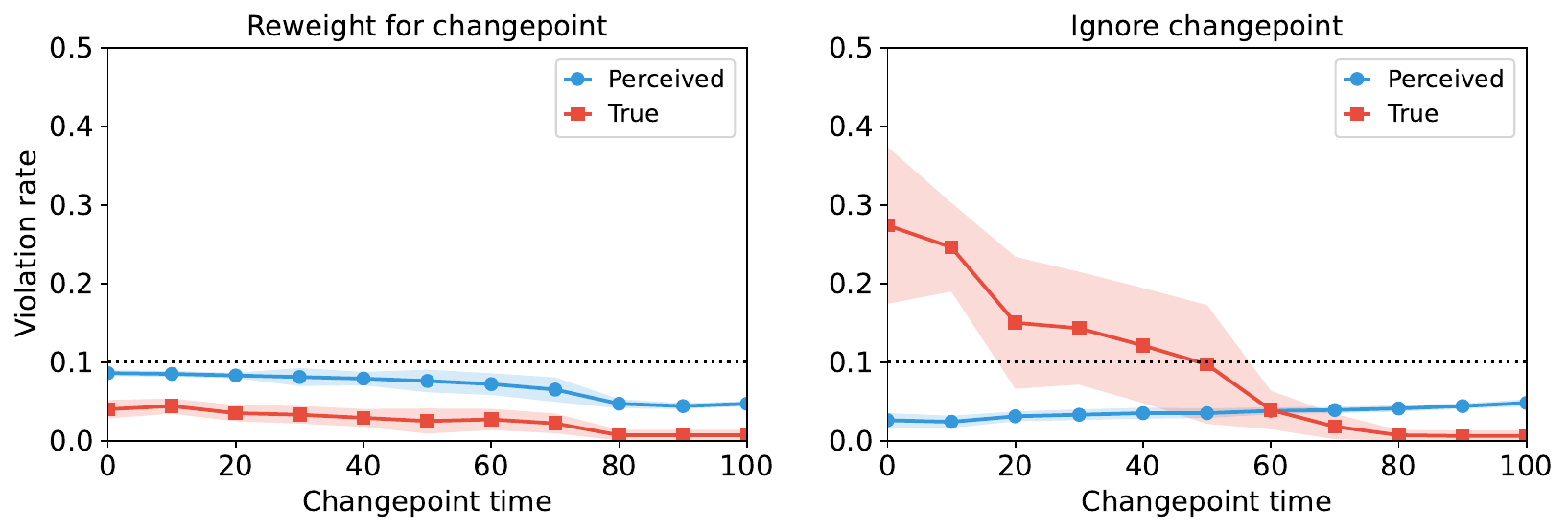} 
\caption{\textbf{Changepoint results.} 
}
\label{fig:changepoint}
\end{figure}

\subsection{Sensitivity analysis}\label{appx:sensitivity-analysis}
When we have noisy data or we have a high-bias or high-variance estimator, conformal coverage still holds. Therefore, the actual and perceived violations are still below $\alpha$. However, the number of standard-of-care fallback steps increases as the estimator worsens or as noise increases (Fig.~\ref{fig:sensitivity-analysis}). For these plots, $\omega=0.1$ since the number of fallback steps is low under default conditions, as shown in Fig.~\ref{fig:synthetic-results}.
\begin{figure}
\centering
\includegraphics[width=\columnwidth]{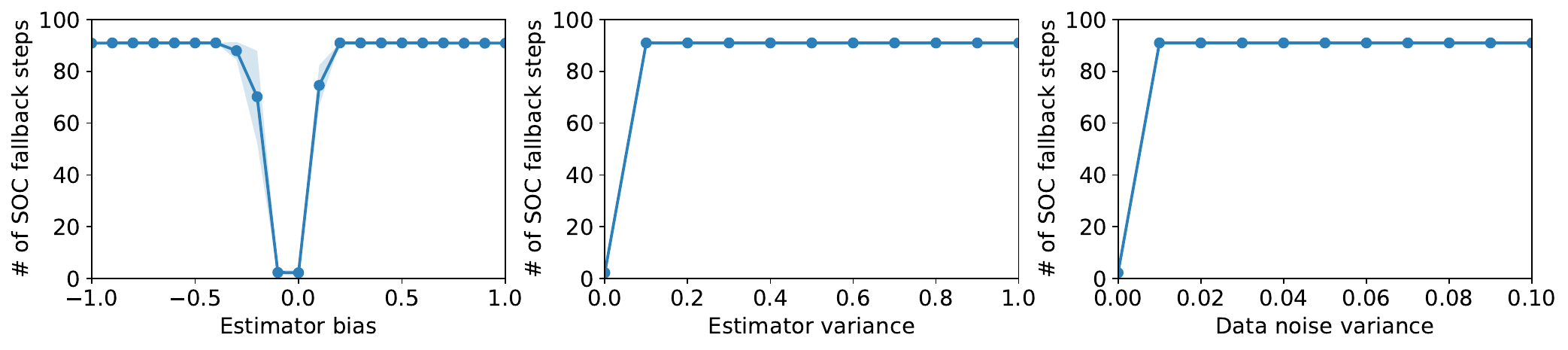} 
\caption{\textbf{Sensitivity analysis: estimator and variability.} 
}
\label{fig:sensitivity-analysis}
\end{figure}

When the weights are mis-specified, conformal coverage does not hold. In Fig.~\ref{fig:sensitivity-weights} we plot the violation rate versus the percent overlap between the training and calibration and the testing covariates. One hundred percent overlap means that both sets have covariates between 0 and 1. Zero percent overlap means that the training and calibration data have covariates between 0 and 1, while the testing data has covariates between 1 and 2. For this plot, $\omega=0.001$ since small changes to the weights have a large impact at low $\omega$ values. 
\begin{figure}
\centering
\includegraphics[width=0.5\columnwidth]{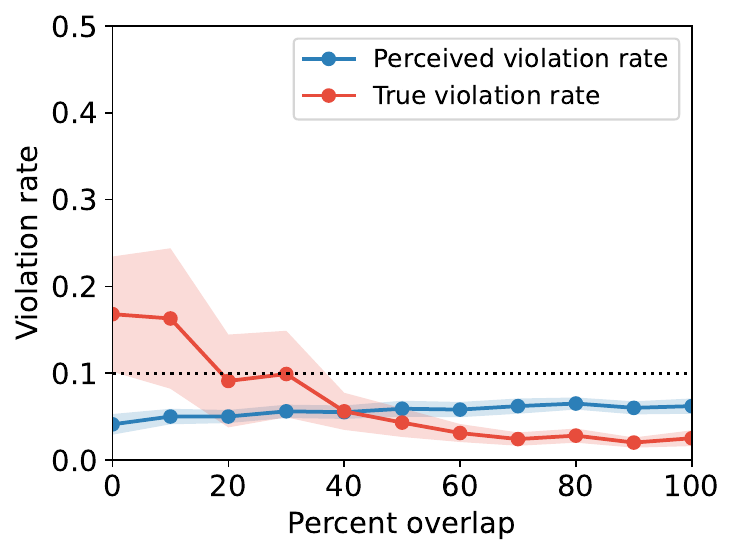} 
\caption{\textbf{Sensitivity analysis: weights.} 
}
\label{fig:sensitivity-weights}
\end{figure}

Note that different $\alpha$ values will change the maximum number of fallback steps and allowed violation rate threshold. High $\omega$ values result in fewer fallback steps in general and therefore less sensitive to noise or worse $\hat{Y}$ estimators, while low $\omega$ values result in more fallback steps. 



\end{document}